\documentclass[lettersize,journal]{IEEEtran}
\usepackage{amsmath,amsfonts}
\usepackage{array}
\usepackage[caption=false,font=normalsize,labelfont=sf,textfont=sf]{subfig}
\usepackage{textcomp}
\usepackage{stfloats}
\usepackage{url}
\usepackage{color}
\usepackage{verbatim}
\usepackage{graphicx}
\usepackage{cite}
\usepackage{amssymb}
\usepackage{booktabs}
\newtheorem{theorem}{Theorem}
\newenvironment{proof}{{\noindent\it Proof. }}{\hfill $\blacksquare$\par}
\usepackage[linesnumbered,algoruled,boxed,lined]{algorithm2e}

\begin{document}
\title{UAV Virtual Antenna Array Deployment for Uplink Interference Mitigation in Data Collection Networks}

\author{Hongjuan Li, Hui Kang, Geng Sun,~\IEEEmembership{Senior~Member,~IEEE}, Jiahui Li, Jiacheng Wang, \\Xue Wang,~\IEEEmembership{Senior~Member,~IEEE}, Dusit Niyato,~\IEEEmembership{Fellow,~IEEE}, and Victor C. M. Leung,~\IEEEmembership{Life~Fellow,~IEEE}

\thanks{This work is supported in part by the National Natural Science Foundation of China (62272194, 62172186, 62471200), in part by the National Key Research and Development Program of China (2022YFB4500600), in part by the Science and Technology Development Plan Project of Jilin Province (20240302079GX), in part by the Postdoctoral Fellowship Program of China Postdoctoral Science Foundation (GZC20240592), in part by China Postdoctoral Science Foundation General Fund (2024M761123), in part by the Graduate Innovation Fund of Jilin University (2024CX318, 2024CX319), in part by the National Research Foundation, Singapore, and Infocomm Media Development Authority under its Future Communications Research \& Development Programme, Defence Science Organisation (DSO) National Laboratories under the AI Singapore Programme (FCP-NTU-RG-2022-010 and FCP-ASTAR-TG-2022-003), in part by the Singapore Ministry of Education (MOE) Tier 1 (RG87/22), in part by the NTU Centre for Computational Technologies in Finance (NTU-CCTF), and in part by Seitee Pte Ltd. Part of this paper appeared in WASA 2022~\cite{Li2022a}. (\emph{Corresponding authors: Geng Sun and Jiahui Li.})}

\thanks{Hongjuan Li is with the College of Computer Science and Technology, Jilin University, Changchun 130012, China (e-mail: hongjuan23@mails.jlu.edu.cn).}
\thanks{Hui Kang is with the College of Computer Science and Technology, Jilin University, Changchun 130012, China, and also with Key Laboratory of Symbolic Computation and Knowledge Engineering of Ministry of Education, Jilin University, Changchun 130012, China (e-mail: kanghui@jlu.edu.cn).}
\thanks{Geng Sun is with the College of Computer Science and Technology, Jilin University, Changchun 130012, China, and with Key Laboratory of Symbolic Computation and Knowledge Engineering of Ministry of Education, Jilin University, Changchun 130012, China; he is also affiliated with the College of Computing and Data Science, Nanyang Technological University, Singapore 639798 (e-mail: sungeng@jlu.edu.cn).}
\thanks{Jiahui Li is with the College of Computer Science and Technology, Jilin University, Changchun 130012, China (e-mail: lijiahui@jlu.edu.cn).}
\thanks{Jiacheng Wang is with the College of Computing and Data Science, Nanyang Technological University, Singapore (e-mail: jiacheng.wang@ntu.edu.sg). }
\thanks{Xue Wang is with the Department of Communication Engineering, Jilin University, Changchun 130012, China (e-mail: txwangxue@jlu.edu.cn).}
\thanks{Dusit Niyato is with the College of Computing and Data Science, Nanyang Technological University, Singapore (e-mails: dniyato@ntu.edu.sg). }
\thanks{Victor C. M. Leung is with the Artificial Intelligence Research Institute, Shenzhen MSU-BIT University, Shenzhen 518115, China, with the College of Computer Science and Software Engineering, Shenzhen University, Shenzhen 518060, China, and also with the Department of Electrical and Computer Engineering, The University of British Columbia, Vancouver V6T 1Z4, Canada (e-mail: vleung@ieee.org).}

\thanks{Copyright (c) 20xx IEEE. Personal use of this material is permitted. However, permission to use this material for any other purposes must be obtained from the IEEE by sending a request to pubs-permissions@ieee.org.}
}

\markboth{Journal of \LaTeX\ Class Files,~Vol.~14, No.~8, August~2021}%
{Li \MakeLowercase{\textit{et al.}}: Uplink Interference Mitigation in Data Collection Networks}


\maketitle

\begin{abstract}
Unmanned aerial vehicles (UAVs) have gained considerable attention as a platform for establishing aerial wireless networks and communications. However, the line-of-sight dominance in air-to-ground communications often leads to significant interference with terrestrial networks, reducing communication efficiency among terrestrial terminals. This paper explores a novel uplink interference mitigation approach based on the collaborative beamforming (CB) method in multi-UAV network systems. Specifically, the UAV swarm forms a UAV-enabled virtual antenna array (VAA) to achieve the transmissions of gathered data to multiple base stations (BSs) for data backup and distributed processing. However, there is a trade-off between the effectiveness of CB-based interference mitigation and the energy conservation of UAVs. Thus, by jointly optimizing the excitation current weights and hover position of UAVs as well as the sequence of data transmission to various BSs, we formulate an uplink interference mitigation multi-objective optimization problem (MOOP) to decrease interference affection, enhance transmission efficiency, and improve energy efficiency, simultaneously. In response to the computational demands of the formulated problem, we introduce an evolutionary computation method, namely chaotic non-dominated sorting genetic algorithm II (CNSGA-II) with multiple improved operators. The proposed CNSGA-II efficiently addresses the formulated MOOP, outperforming several other comparative algorithms, as evidenced by the outcomes of the simulations. Moreover, the proposed CB-based uplink interference mitigation approach can significantly reduce the interference caused by UAVs to non-receiving BSs.
\end{abstract}


\section{Introduction}
\label{sec:introduction}
\IEEEPARstart{D}{ue} to their significant benefits of flexibility and autonomy, the utilization of unmanned aerial vehicles (UAVs) in military scenarios has become prevalent~\cite{Zeng2019}. In recent years, civilian and commercial UAVs have gained widespread applications with the reduction in manufacturing costs and advancements in manufacturing technology~\cite{Liu2024}. Among various UAV-enabled technologies, UAV wireless network and communication technologies have attracted more and more attention from domestic and international scholars~\cite{Li2023, Du2024}. For example, UAVs can be designed as aerial relays for offering data forwarding services to several ground users (GUs) that lack reliable links~\cite{Sun2022}. Moreover, UAVs can be adopted to achieve the integration of sensing and communication functions by carrying relevant sensors, thereby improving the utilization rate of radio resources~\cite{Deng2023}. Furthermore, UAVs can be utilized as data collection platforms for the Internet of Things (IoT) systems~\cite{AbdelBasset2024}, in which the line-of-sight (LoS) dominant channel can reduce signal propagation loss and enhance data transmission performance.
\par Although the LoS dominant channel offers several advantages, UAV-based data collection still struggles with some critical challenges, including limited onboard energy and restricted transmission power~\cite{Wu2020}, which can limit the long-range communication capability of UAVs. In addition, UAV networks may also cause severe interference to terrestrial devices. Furthermore, numerous interference mitigation techniques, including trajectory design and transmission power control, have been proposed to address the interference issues~\cite{Mei2019, Liu2021}.
\par Among them, the power control-based interference mitigation method can reduce interference to terrestrial networks by controlling the transmission powers of UAVs. Specifically, this approach limits the power of the UAV to a level where the targeted receiver has full reception and low interference with other communication devices. In this case, the data transmission rate of UAVs can be decreased~\cite{Yajnanarayana2018}. On the other hand, the trajectory design-based interference mitigation method refers to designing the UAV flight paths in accordance with the distribution of terrestrial terminals, so that they can transmit data at the positions far from these terminals. However, changing the positions of UAVs may lead to a longer transmission distance to the targeted receiver, thus increasing the data transmission time of UAVs~\cite{Lee2022}. Therefore, there is an urgent need to propose an innovative interference mitigation approach aimed at minimizing interference and lessening transmission time, simultaneously.
\par Collaborative beamforming (CB) is a critical method in the field of collaborative communications, aiming to remarkably enhance the data transmission capability of a single element by utilizing the cooperative work of multiple elements. In the UAV-enabled communication systems, CB can be accomplished by deploying a virtual antenna array (VAA), which consists of omnidirectional antennas equipped on multiple UAVs to simulate the effect of an antenna array. By jointly designing hover positions and excitation current weights of the UAV swarm, the signal strength and direction can be precisely controlled, thereby achieving beamforming with high directivity and high gain~\cite{Sun2024}. Based on these properties, the efficacy of data transmission is significantly enhanced through the optimization of the directivity of the mainlobe. For example, in~\cite{Li2024}, the authors applied the CB method to reduce the time in UAV-assisted data harvesting and dissemination. Moreover, the interference to terrestrial terminals can be mitigated since the signal intensity in directions other than the mainlobe is suppressed. Thus, the dual goals of transmission performance improvement and interference mitigation in UAV-enabled data collection networks can be achieved via CB.
\par The position of the UAV is the crucial element that impacts the performance of CB. In other words, UAVs must move the optimal positions for performing the CB. However, more energy can be consumed when UAVs frequently change their positions, which can shorten the lifetime of UAV-enabled data collection networks. As a result, we formulate a multi-objective optimization problem (MOOP) aimed at striking a balance between the interference mitigation capabilities and energy efficiency of UAVs. This balance is essential to satisfy diverse requirements for interference mitigation and energy consumption across various scenarios.
\par The following outlines the key contributions of this work:
\begin{itemize}
  \item \emph{\textbf{CB-based Uplink Interference Mitigation Method:}} We introduce the CB method in multi-UAV systems and propose a novel approach for mitigating interference to terrestrial networks during uplink data transmission. Unlike existing approaches that rely on power control or trajectory design, the CB-based method effectively reduces interference with non-receivers while improving the transmission efficiency of UAVs.
  \item \emph{\textbf{Multi-Objective Optimization Problem:}} We formulate an uplink interference mitigation MOOP for optimizing the transmission efficiency, increasing the signal-to-interference-plus-noise ratio (SINR) of the interfered base stations (BSs), and minimizing propulsion energy consumed by UAVs, simultaneously. Furthermore, we prove that the MOOP is NP-hard, highlighting its computational complexity.
  \item \emph{\textbf{Chaotic Evolutionary Computation Algorithm:}} We propose a chaotic non-dominated sorting genetic algorithm II (CNSGA-II) to deal with the complex MOOP. This algorithm integrates chaos theory into solution initialization, crossover, and mutation operators, improving the quality of the initial population and balancing exploration and exploitation. Additionally, CNSGA-II employs an enhanced elimination strategy to discard non-competitive solutions and improve convergence accuracy.
  \item \emph{\textbf{Extensive Simulation Analysis:}} We perform comprehensive simulations to assess the performance of the CB-based approach in mitigating interference. The results indicate that our method effectively reduces interference to terrestrial networks and outperforms several comparison algorithms.
\end{itemize}
\par The remainder of this work is structured as follows. The related works are presented in Section \ref{sec:related_work}. Section \ref{sec:system_model} introduces the related models utilized in this paper. Section \ref{sec:problem_formulation_and_analysis} formulates a MOOP to simultaneously reduce the interference of the UAVs to non-receiving BSs and minimize the transmission time and propulsion energy consumption. Section \ref{sec:the_proposed_method} proposes the CNSGA-II. Section \ref{sec:simulation_results_and_analysis} provides simulations and this work is summarized in Section \ref{sec:conclusion}.

%
%
\section{Related Work}
\label{sec:related_work}
\par This section provides a brief review of prior works that focus on interference mitigation methods and CB.
\subsection{Interference Mitigation}
\par Most existing works on interference mitigation in wireless networks focused on reducing interference by controlling the transmission power, designing the flight trajectory, or optimizing the channel selection of transmitters.
\par Some works aimed to reduce interference in wireless networks by limiting the transmit power to a level that ensures reliable reception by the target receiver and low interference with other communication devices. For example, Li \emph{et al}.~\cite{Li2019} adopted a mean field game to represent the interference mitigation problem and then proposed an optimal power control method to solve the problem. Moreover, Qi \emph{et al}.~\cite{Qi2024} designed a jamming strategy that maximizes interference on hostile users while controlling mutual interference by jointly optimizing the power and position of the jammer and the transmit powers of allies.
\par Several works optimized the flight path of UAVs based on the spatial distribution of terrestrial terminals for enabling data transmission from locations farther away from these terminals. In~\cite{Liu2024a}, the authors designed a UAV trajectory planning method based on block coordinate descent, which aims to effectively reduce the UAV task completion time while meeting the time and energy consumption constraints and mitigating interference with the terrestrial equipment. In addition, Lee \emph{et al}.~\cite{Lee2022} proposed an interference-aware path planning optimization approach aimed at minimizing co-channel interference among UAVs while serving GUs by optimizing UAV flight paths.
\par Several works considered jointly optimizing the trajectory and power of UAVs for interference mitigation. For example, Shen \emph{et al}.~\cite{Shen2020} designed a UAV-enabled interference channel by jointly optimizing trajectory and power control to decrease the cross-link interference in UAV networks. Moreover, in~\cite{Burhanuddin2023}, the authors designed a downlink inter-cell interference management method based on the deep-Q network to reduce the interference effect between BSs and Gus while meeting the transmission rate requirements of UAVs. Furthermore, Wang \emph{et al}.~\cite{Wang2022} investigated an interference coordination strategy for UAV-relaying in 5G/6G spectrum sharing networks, in which the authors aimed to coordinate mutual interference among ground nodes and improve energy efficiency by jointly optimizing UAV transmit power and trajectory.
\par Recently, several works have been proposed to achieve interference mitigation by optimizing channel selection. For instance, in~\cite{Su2024}, the authors explored a channel selection-based approach for anti-jamming and interference mitigation in UAV networks for addressing challenges from both malicious jamming and co-channel interference. Furthermore, Zhou \emph{et al}.~\cite{Zhou2021} investigated a novel resource allocation strategy for multi-UAV systems to combat adjacent and co-channel interference, in which the minimum SINR across UAVs was maximized by optimizing both channel and power allocation. In addition, Vaezi \emph{et al}.~\cite{Vaezi2024} pointed out that deep reinforcement learning (DRL) can be applied to mitigate interference in UAV-assisted cellular networks without accurate channel state information (CSI) by adjusting transmit power and dynamically selecting communication channels.
\par We summarize the differences from the above-mentioned related studies as follows. \emph{First}, most of these methods consider a power allocation or trajectory optimization approach to reduce interference, which may decrease the transmission efficiency of the data collection and transmission process. \emph{Second}, the effectiveness of some channel selection-based approaches is constrained by the number of available channels and the need for frequent channel sensing and allocation, which leads to higher computational overhead and increased energy consumption. \emph{Third}, the above works do not consider adopting multi-objective optimization to find the different trade-offs between interference performance, data transmission time, and energy consumption.
\subsection{Collaborative Beamforming}
\par There are published works that take the CB method into account for obtaining high-performance transmission in different communication systems. For example, Sun \emph{et al}.~\cite{Sun2022} introduced some typical CB applications in UAV networks and proposed two schemes for UAV communication based on CB for achieving high-performance communication. Moreover, Li \emph{et al}.~\cite{Li2024} proposed an innovative scheme to achieve time and energy-saving harvesting and dissemination of IoT data by integrating UAV technology and the CB method. In addition, Sun \emph{et al}.~\cite{Sun2022a}studied the problem of UAV-enabled relay communication and proposed an improved evolutionary computation approach for achieving secure and energy-efficient communication to multiple GUs. Furthermore, in~\cite{Li2022}, the authors utilized CB to address the long-distance and energy-saving uplink transmission problem in UAV-assisted mobile wireless sensor networks. Additionally, Li \emph{et al}.~\cite{Li2024a} explored an uplink communication approach from terrestrial terminal to satellite based on the distributed CB. Although the above works can effectively solve some problems in wireless networks by applying the CB method, they ignore a key issue, namely the interference of the data propagation process to non-target receivers.

\par Some emerging techniques, such as movable antenna (MA) and integrated sensing and communication, have been applied to improve the performance of wireless communication systems, in which similar to the UAV-enabled VAA used in the CB method, the MA method aimed to optimize signal reception and transmission by adjusting the positions of multiple antennas to reconstruct the array geometry~\cite{Zhu2024}. For example, Ma \emph{et al}.~\cite{Ma2024} explored the application of MA in multiple-input multiple-output (MIMO) communication systems, in which the channel power and the condition number of the MIMO channel can be improved and reduced, respectively. Moreover, in~\cite{Zhu2024a}, the authors evaluated the performance of MA in wireless communication systems and demonstrated that MA can significantly enhance signal strength and improve communication reliability. Furthermore, Hu \emph{et al}.~\cite{Hu2024} explored MA array-assisted physical layer secure communications, in which the authors focused on enhancing secrecy performance by optimizing the transmit beam pattern and the positions of MAs at Alice. Additionally, Cheng \emph{et al}.~\cite{Cheng2024} investigated a novel secure transmission method, which aimed to achieve enhanced physical layer security by jointly designing the beamformer and the positions of MAs. In addition, Zhu \emph{et al}.~\cite{Zhu2023} studied enhanced beamforming of MA by adjusting the positions and weights of antenna elements, which aimed to maximize array gain in the desired direction while suppressing interference in undesired directions. Similarly, Ren \emph{et al}.~\cite{Ren2024} explored integrating the six-dimensional MA array into cellular-connected UAV systems for interference mitigation and enhanced communication.
\par We summarize the differences between CB and MA as follows. \emph{First}, the MA method does not require inter-UAV data sharing or synchronization, simplifying system design and reducing communication overhead, while offering more stable beamforming due to precise mechanical adjustments. However, it involves higher equipment costs and complex mechanical integration, making it less cost-effective compared to the simple omnidirectional antennas used in UAV-enabled VAA systems. \emph{Second}, MAs have lower control complexity since they can operate independently on a single UAV. In contrast, UAV-enabled VAA requires precise coordination of multiple UAVs to form a virtual array, which increases deployment complexity but provides greater flexibility. For example, the number of UAVs forming the VAA can be dynamically adjusted based on mission requirements, such as 8 or 16 UAVs, to adapt to diverse scenarios. \emph{Third}, while MAs are effective in improving communication efficiency and mitigating interference, their limited flexibility and scalability restrict their coverage in large-scale data collection tasks. In comparison, UAVs in the VAA can move freely when not forming the array, enabling them to cover larger areas for monitoring and data collection.

\begin{table}[tbp]
    \centering
    \caption{Main notations and descriptions.}
    \begin{tabular}{p{2cm} p{6cm}}
    \hline
    \bf{Notation} & \bf{Description}\\
    \hline
    $\mathcal{U}$ & The set of UAVs.\\
    $\mathcal{B}$ & The set of BSs.\\
    $A_{m}$ & Monitoring area.\\
    $\mathcal{F}(\theta, \phi)$ & Array factor.\\
    $I_{i}$ & Excitation current weight. \\
    $\lambda$ & wavelength.\\
    $P_{LoS}$ & The probability of LoS link. \\
    $P_{NLoS}$ & The probability of NLoS link.\\
    $\theta_{e}$ & Elevation angle between VAA and targeted BS.\\
    $k_{0}$ & Pathloss constant.\\
    $\alpha$ & Pathloss exponent.\\
    $d_{j}$ & The distance from the VAA to the $j$BS.\\
    $\xi_{1}$ & The attenuation coefficients of LoS link.\\
    $\xi_{2}$ & The attenuation coefficients of NLoS link.\\
    $g_{BS_{j}}$ & Channel power gain toward the $j$th BS.\\
    $P_{t}$ & Total transmission power of the VAA.\\
    $\mathcal G_{BS_{j}}$ & Antenna gain of the VAA towards the $j$th BS.\\
    $P_{BS_{j}}$ & Receiver power of the $j$th BS from the VAA.\\
    $B$ & Transmission bandwidth.\\
    $\sigma^2$ & Noise power.\\
    $\mathcal R_{BS_{j}}$ & Transmission rate from the VAA to the $j$th BS.\\
    $\mathcal {Y}_{BS_{in}}$ & The SINR of the interfered BS.\\
    $P_1$ & Blade profile power when the UAV keeps hovering. \\   
    $P_2$ & Blade induced power when the UAV keeps hovering. \\
    $v_{tip}$ & Tip speed of the rotor blade. \\
    $v_{0}$ & Average rotor-induced speed during hovering. \\
    $d_0$ & Fuselage drag ratio. \\
    $s$ & Rotor solidity. \\
    $A$ & Rotor disc area. \\
    $\rho$ & Air density. \\
    $V$ & Flying speed of UAVs.\\
    $\mathcal {P}\left(V\right)$ & Energy consumption in two-dimensional (2D) plane generated by UAV.\\
    $\mathcal E(T)$ & Energy consumption in three-dimensional (3D) space generated by UAV.\\
    $m_{u}$ & The mass of UAV.\\
    \hline
    \end{tabular}
    \label{table:symbols}
\end{table}
%
%
\section{System Model}
\label{sec:system_model}
\par In this section, we outline some of the key models utilized in this work. Moreover, Table~\ref{table:symbols} shows the main notations used in the system model and their corresponding descriptions.
\subsection{Network Model}
\par Fig. \ref{fig:system-model} shows a UAV-enabled data collection and air-to-ground (A2G) transmission scenario, where $N_{U}$ UAVs ($\mathcal U = \{1, 2, \dots, N_{U}\}$) are deployed at random in the monitoring area represented as $A_m$ for performing data collection task, and the flights of these UAVs cannot exceed the monitoring area. Meanwhile, some BSs ($\mathcal B = \{1, 2, \dots, N_{B}\}$) that are remote from $A_m$ are used to receive and process data from the UAVs. When the task is completed or the collected data reaches the cache limit, these UAVs will upload the obtained data to BSs for analysis and preservation. However, other BSs may be communicating with GUs during the process of the UAVs transmitting data to BS $i$, which means that these data transmission processes will decrease the communication performance between BSs and GUs. In particular, since the high flight altitude of UAVs and high LoS probability of A2G transmissions, UAVs will cause a wider range and stronger intensity of interference to terrestrial networks. Thus, these UAVs form a UAV-enabled VAA for transmissions with BSs through the CB. Different from~\cite{Li2022a} which only considers one target BS, in the scenario considered in this work, UAVs transmit the obtained data to multiple target BSs. This is because data transmission to multiple BSs is imperative for data backup and data security in some cases~\cite{Zhang2023}. Moreover, multiple target BSs can achieve data sharing and decentralized processing, which can meet some practical scenario requirements that require rapid response.
\label{ssec:network_model} 
\begin{figure}
	\centering
	\includegraphics[width=3.5in]{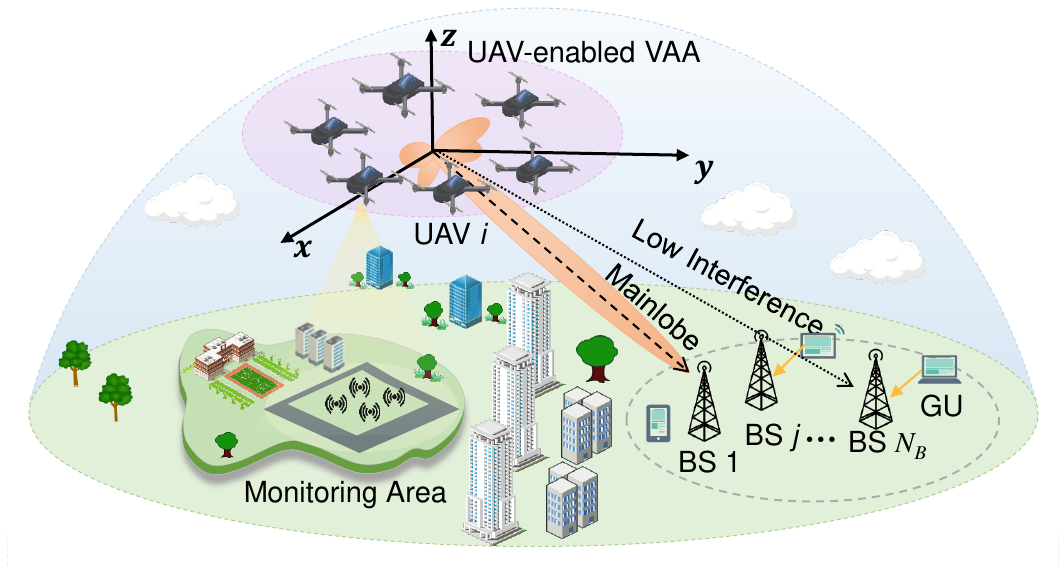}
	\caption{Sketch map of a UAV-enabled VAA model for CB.}
	\label{fig:system-model}
\end{figure}
\par For ease of expression, a 3D Cartesian coordinate system is utilized, in which $(x^{B}_{j}, y^{B}_{j}, 0)$ and $(x^{U}_{i}, y^{U}_{i}, z^{U}_{i})$ respectively indicate the positions of the $j$th BS and the $i$th UAV.
\subsubsection{Array Factor (AF) of UAV-enabled VAA}
\label{ssec:array_factor_of_UVAA}
\par Single omni-directional antennas equipped on multiple UAVs form a UAV-enabled VAA to enhance the signal towards the targeted BS, where AF can be used to show the signal strength of the VAA in all directions, and AF is modeled as follows~\cite{Sun2021}:
\begin{equation}
   \begin{aligned}
      \mathcal F(\theta, \phi)=
      \sum \limits _{i = 1}^{N_{U}} I_{i} e^{k \left [ 2\pi/\lambda \left ({{{x^{U}_{i}\sin \theta \cos \phi + {y^{U}_{i}}\sin \theta \sin \phi + {z^{U}_{i} }}\cos \theta } }\right)\right]},
   \end{aligned}
   \label{eq:AF}
\end{equation}
\noindent where $I_{i}$ and $(x^{U}_{i}, y^{U}_{i}, z^{U}_{i})$ respectively indicate the excitation current weight and position coordinate of UAV $i$. Moreover, $2\pi/\lambda$ denotes the constant related to phase and $\lambda$ is the wavelength.
\par In our considered scenario, we assume that data synchronization between UAVs performing the VAA is accomplished through the data-sharing protocol proposed in~\cite{Feng2013}. In this process, we can select the UAV with the most remaining energy as the master node to collect and integrate the data obtained by other UAVs, and then the selected master node broadcasts the final data packet to all UAVs. Furthermore, it is noted that the timing, phase, and frequency synchronization of these UAVs can be completed by using the synchronization protocols and techniques described in~\cite{Mohanti2019},~\cite{Boyle2017}, and~\cite{Alemdar2021}. Based on this, the backhaul connection between these UAVs can achieve reliable baseband data transmission and successfully construct the desired beam patterns.
\par Moreover, some existing CSI evaluation works specifically proposed for CB can be adopted to achieve CSI acquisition~\cite{Zhang2024, Ahmad2022}. Specifically,~\cite{Ahmad2022} proposed a codebook-based channel quantization method for CB, which can use a small number of bits to achieve performance comparable to perfect channel feedback. In this case, the imperfect CSI due to CSI errors will not cause much performance loss to such CB-based approaches~\cite{Li2024}. Thus, we assume that the UAV-enabled VAA can embed this method, and then use the CSI quantization results to determine their transmission parameters, such as excitation current weights and residual phases. Note that we select a UAV to perform and coordinate the entire optimization process and the details will be introduced in Section \ref{ssec:Optimization_Strategy}.

\subsubsection{Channel Model}
\label{ssec:channel_model}
\begin{figure}
	\centering
	\includegraphics[width=3.0in]{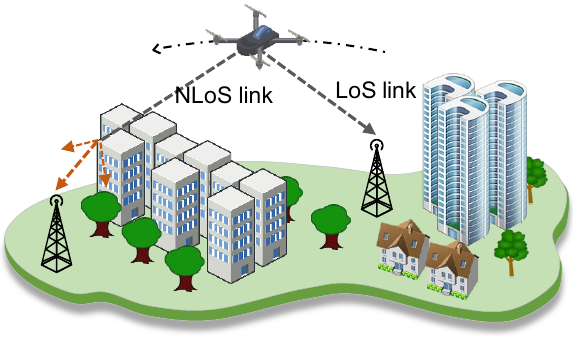}
	\caption{System model of an A2G transmission in the probabilistic LoS channel.}
	\label{fig:LoS-NLoS}
\end{figure}
\par Fig. \ref{fig:LoS-NLoS} shows an A2G transmission scenario in the probabilistic LoS channel, where the LoS links between UAVs and terrestrial BSs may be blocked by buildings. To better fit the practical real scenarios, we adopt the channel model used in~\cite{Zeng2019} and~\cite{Chen2017}, and the probability of LoS link is modeled as follows~\cite{Sun2022}:
\begin{equation}
   \begin{aligned}
      P_{L o S}= [1+k_{1} exp(-k_{2}(\theta_{e}-k_{1}))]^{-1},
   \end{aligned}
   \label{eq:P_LoS}
\end{equation}
\noindent where $k_{1}$ and $k_{2}$ are parameters related to the environment. In addition, $\theta_{e}$ is the elevation angle between the VAA and the targeted BS, which is given by
\begin{equation}
   \begin{aligned}
      \theta_{e} = 180/\pi\times\sin^{-1}(h_{u}/d_{2D}),
   \end{aligned}
   \label{eq:theta_e}
\end{equation}
\noindent where $d_{2D}$ and $h_{u}$ denote the 2D horizontal distance and vertical altitude from the UAV-enabled VAA to targeted BS, respectively.
\par According to Eq. (\ref{eq:P_LoS}), the non-LoS (NLoS) probability $P_{NLoS}$ can be given as $P_{NLoS} = 1 - P_{LoS}$. Therefore, the channel power gain is defined as follows:
\begin{equation}
   \begin{aligned}
      g_{B S_{j}}=[k_{0}d_{j}^{\alpha}(\xi_{1}P_{L o S}+\xi_{2}P_{N L o S})]^{-1},
   \end{aligned}
   \label{eq:g_c}
\end{equation}
\noindent where $k_{0}$ and $\alpha$ are the pathloss constant and the pathloss exponent, respectively. Moreover, $d_{j}$ indicates the distance from the VAA to BS $j$. In addition, $\xi_{1}$ and $\xi_{2}$ respectively denote the attenuation coefficients of LoS and NLoS links. Note that this probabilistic LoS model takes into account the obstacles and environmental changes in the channel, making the model closer to the actual situation.
\subsubsection{Transmission and Interference Model}
\label{ssec:transmission_and_interference_model}
\par Based on the probabilistic LoS channel model introduced above, the receiver power of BS $j$ from the UAV-enabled VAA is given by
\begin{equation} 
    \begin{aligned}
       P_{BS_{j}}= \frac{P_t\mathcal G_{BS_{j}}}{g_{B S_{j}}},
    \end{aligned}  
    \label{eq:receiver-power} 
\end{equation}
\noindent where $P_{t}$ is the total transmission power of the VAA. $\mathcal G_{BS_{j}}$ indicates the antenna gain of the VAA towards BS $j$, which is modeled as follows~\cite{Mozaffari2019}:
\begin{equation} 
    \begin{aligned}
       \mathcal G_{BS_{j}}=\frac{4 \pi\left|\mathcal F\left(\theta_{B S_{j}}, \phi_{B S_{j}}\right)\right|^{2} \omega\left(\theta_{B S_{j}}, \phi_{B S_{j}}\right)^{2}}{\int_{0}^{2 \pi} \int_{0}^{\pi}|\mathcal F(\theta, \phi)|^{2} \omega(\theta, \phi)^{2} \sin \theta d \theta d \phi} \eta,
    \end{aligned}  
    \label{eq:G} 
\end{equation}
\noindent where $\eta \in [0,1]$ denotes the efficiency of antenna array. Moreover, $(\theta_{B S_{j}}, \phi_{B S_{j}})$ represents the direction towards the position of the $j$th BS, and $\omega(\theta, \phi)$ indicates the magnitude of the far-field beam pattern of each UAV.
\par According to Eq. (\ref{eq:receiver-power}), the transmission rate from the VAA to targeted BS $j$ is expressed as follows~\cite{Mozaffari2019a}:
\begin{equation} 
    \begin{aligned}
       \mathcal R_{BS_{j}}=B \log_2\left(1+\frac{P_{BS_{j}}}{\sigma^2}\right),
    \end{aligned}  
    \label{eq:Transmission-rate-BS} 
\end{equation}
\noindent where $\sigma^2$ is noise power, and $B$ indicates the transmission bandwidth.
\par When the VAA is communicating with the targeted BS, other BSs may be communicating with some GUs, in which the targeted power received by these BSs from GUs is $P_{GU}$. Note that the $P_{GU}$ is set as a constant for ease of calculation. Based on the broadcast characteristics of wireless communication and LoS dominant channel features, these BSs also receive the signal power $P_{BS_{in}}$ from the VAA, and for these BSs, $P_{BS_{in}}$ is the interference power. Consequently, the SINR of the interfered BS can be defined as follows~\cite{Mozaffari2016}:
\begin{equation} 
    \begin{aligned}
       \mathcal {Y}_{BS_{in}}=\frac{{P}_{GU}}{\sigma^2+ {P}_{BS_{in}}}.
    \end{aligned}  
    \label{eq:Transmission-SINR-BS} 
\end{equation}

\subsection{Energy Consumption Model}
\label{ssec:energy_consumption_model_of_UAV}
\par As illustrated in~\cite{Zeng2019a}, the communication-related energy consumption of UAVs is much smaller than propulsion energy consumption. In this case, there is no need to account for the communication overhead of UAVs. According to Eq. (12) proposed in~\cite{Zeng2019a}, the energy consumption in a 2D horizontal plane for UAVs  is directly proportional to their flying speed $V$, which can be given by
\begin{equation}
  \label{eq:energy-2d}
   \begin{aligned}
      \mathcal {P}\left(V\right)&=P_1\left(1+\frac{3V^2}{v^{2}_{tip}}\right)+P_2\left(\sqrt{1+\frac{V^4}{4v^4_0}}-\frac{V^2}{2v^2_0}\right)^{1/2}\\
      &+\frac{1}{2} \rho d_0sAV^3,
   \end{aligned}
\end{equation}
\noindent where $P_{1}$ and $P_{2}$ denote the blade profile power and induced power when the UAV maintains hovering, respectively. In addition, $v_{tip}$ and $v_{0}$ represent the tip speed of the rotor blade and the average rotor induced speed during hovering, respectively. Moreover, $d_{0}$ is the fuselage drag ratio, $s$ and $A$ are rotor solidity and rotor disc area, respectively, and $\rho$ is the air density.
\par Furthermore, the propulsion energy consumption in 3D space for UAVs can be given based on Eq. (23) described in~\cite{Zeng2019}, which is given by
\begin{equation}
  \label{eq:energy-3d}
   \begin{aligned}
     \mathcal E(T) \approx \int_0^T \mathcal P(V_{t})dt+{\frac12m_{U}(V_{t}^2- V_{0}^2)}+{m_{U}gh },
   \end{aligned}
\end{equation}
\noindent where $\mathcal P(\cdot)$ is given by Eq. (\ref{eq:energy-2d}), and $V_t$ denotes the flight speed of the UAV at the moment $t$. Moreover, $m_U$ is the mass of the UAV, $h$ represents the change in the altitude of the UAV during the flight, and $g$ indicates the gravitational acceleration. In addition, $T$ denotes the end time of the UAV motion.
\par Note that as suggested in~\cite{Yang2020, Chou2020}, the oblique flight of a UAV is modeled as two phases which are horizontal motion and vertical motion. This is due to the absence of accurate energy consumption models for UAVs in 3D space~\cite{Zeng2019}. Moreover, rotary-wing UAVs need precise control to maintain flight stability, and oblique flights may introduce operational uncertainty.
%
%
\section{Problem Formulation and Analysis}
\label{sec:problem_formulation_and_analysis}
\par In our considered system, highly efficient and low-interference communication links between UAVs and terrestrial BSs need to be established. Specifically, the UAV swarm should establish contact with terrestrial BSs as soon as possible to transmit the data collected in area $A_m$. Moreover, owing to energy limitations, the UAV swarm cannot directly fly to distant BSs for data transmission. As a result, the UAV swarm forms a VAA to establish communication with various BSs via CB. Note that this considered system permits the transmission of diverse data types to various BSs, and the VAA can also send the same data for data sharing~\cite{Zhang2021}. However, since the high altitude of UAVs and the communication links with high LoS probability, the interference of the VAA to terrestrial networks has to be taken into account.
\par According to the basic principle of the CB method, we can improve transmission efficiency while reducing interference by enhancing the UAV-enabled VAA beam pattern. UAVs need to move the best positions and adjust to the best excitation currents for enhancing the VAA performance, \emph{i.e}., the directionality of the high-gain mainlobe to the targeted BS needs to be improved. However, the changes in the positions of UAVs generate extra energy consumption. Moreover, the mainlobe of the VAA can only be directed in a single direction at a time, which means that UAVs need to adjust their positions after each data transmission with the targeted BS is completed for communication with the next BS. Thus, a large amount of propulsion energy consumption will be generated for transmission with multiple BSs. To increase the lifetime of the considered system, energy-saving also should be considered, which can be achieved by jointly optimizing the positions of the UAVs and the communication sequence with multiple BSs.
\begin{table*}
\caption{Notations of the decision variables.\label{table:notations}}
    \centering
    \begin{tabular}{p{2.7cm} p{4cm} p{5cm} p{4.5cm}}
	\toprule
	\textbf{Symbol}	& \textbf{Variable elements}	& \textbf{Descriptions} & \textbf{Example}\\
	\midrule
	$\mathbb I^{\mathcal B \times \mathcal U}$
	& $\{ I_{j,i} | \forall j \in \mathcal{B}, \forall i \in \mathcal{U}\}$
	& $\mathbb I^{\mathcal B \times \mathcal U}$ is the set of excitation current weights of UAVs when communicating with different BSs, while $I_{j,i}$ is the excitation current weight of UAV $i$ for transmitting data to BS $j$.
	& $I_{1,6} = 0.7$ represents the excitation current weight of the 6th UAV when transmitting data to the 1st BS is 0.7.
	\\
	\midrule
	$(\mathbb X^{\mathcal B \times \mathcal U}, \mathbb Y^{\mathcal B \times \mathcal U}, \mathbb Z^{\mathcal B \times \mathcal U})$
	& $\{ x_{j,i} / y_{j,i} / z_{j,i} | \forall j \in \mathcal{B}, \forall i \in \mathcal{U}\}$
	& $(\mathbb X^{\mathcal B \times \mathcal U}, \mathbb Y^{\mathcal B \times \mathcal U}, \mathbb Z^{\mathcal B \times \mathcal U})$ is the set of 3D position coordinates of UAVs when communicating with different BSs, while $(x_{j,i},y_{j,i},z_{j,i})$ is the position of UAV $i$ for transmitting data to BS $j$.
	& $(x_{1,6},y_{1,6},z_{1,6}) = (30,30,85)$ represents the position of the 6th UAV when transmitting data to the 1st BS is (30,30,85).
	\\
	\midrule
	$\mathbb Q^{\mathcal B \times 1}$
    & $\{ Q_{1} \dots Q_{N_{B}}| \forall j \in \mathcal{B}, \forall Q_{j} \in \mathcal{B}\}$
	& $\mathbb Q^{\mathcal B \times 1}$ is the sequence that the UAV-enabled VAA transmits the collected data to different BSs, while $Q_{j}$ is the $j$th communicating $Q_{j}$ BS.
	& $\mathbb Q^{\mathcal B \times 1} = (2,3, \dots,6)$ means that the UAV-enabled VAA will transmit data to these BSs in accordance with the order of BS2, BS3, ..., BS6.
	\\
	\bottomrule
	\end{tabular}
\end{table*}
\par Based on the above analysis, we formulate a MOOP with optimization objectives summarized as follows. Moreover, Table \ref{table:notations} provides the related information regarding the variable symbols.
\par As such, the decision space consists of two parts which are continuous dimensions and discrete dimensions. Specifically, the excitation current weights $\mathbb I^{\mathcal B \times \mathcal U}$ and the 3D position coordinates $(\mathbb X^{\mathcal B \times \mathcal U}, \mathbb Y^{\mathcal B \times \mathcal U}, \mathbb Z^{\mathcal B \times \mathcal U})$ of UAVs are continuous decision variables. The directivity and strength of the mainlobe can be enhanced by optimizing these continuous variables to achieve efficient and low-interference data transmission to a specific target BS. In addition, since the UAV swarm needs to transmit the collected data to multiple target BSs for data backup and decentralized processing, we need to optimize the communication order between the UAV swarm and the target BSs, which is determined by the discrete decision variables $\mathbb Q^{\mathcal B \times 1}$. By jointly optimizing the discrete variables $\mathbb Q^{\mathcal B \times 1}$ and the continuous variables $\mathbb I^{\mathcal B \times \mathcal U}$ and $(\mathbb X^{\mathcal B \times \mathcal U}, \mathbb Y^{\mathcal B \times \mathcal U}, \mathbb Z^{\mathcal B \times \mathcal U})$, the time cost, total SINR, and energy consumption of the entire system are minimized, thereby improving the overall performance of the considered system.
\par \emph{\textbf{Optimization Objective 1:}} To extend their effective working time, UAVs need to transmit data as quickly as possible. Thus, the first optimization objective is given by:
\begin{equation}
\label{eq:optimization_objective_1}
   \begin{aligned}
     f_1 (\mathbb I^{\mathcal B \times \mathcal U}, \mathbb X^{\mathcal B \times \mathcal U}, \mathbb Y^{\mathcal B \times \mathcal U}, \mathbb Z^{\mathcal B \times \mathcal U}, \mathbb Q^{\mathcal B \times 1}) = \sum_{j =1}^{N_{B}} \frac{Data_{j}}{\mathcal{R}_{BS_{j}}},
   \end{aligned}
\end{equation}
\noindent where $\boldsymbol X = \{\mathbb I^{\mathcal B \times \mathcal U}, \mathbb X^{\mathcal B \times \mathcal U}, \mathbb Y^{\mathcal B \times \mathcal U}, \mathbb Z^{\mathcal B \times \mathcal U}, \mathbb Q^{\mathcal B \times 1}\}$ denotes the solution of the MOOP, and $Data_{j}$ is the total data that the VAA transmits to the $j$th BS.

\par \emph{\textbf{Optimization Objective 2:}} Except for the data transmission from the VAA to the selected BS, other BSs also need to communicate with some GUs, which may be interfered with by the data transmission process. Thus, the total SINR of the interfered BSs is used to describe the interference mitigation effect of the CB-based strategy, in which the larger the total SINR the better the interference mitigation performance, which can be given by
\begin{equation}
  \label{eq:optimization_objective_2}
   \begin{aligned}
     f_2 (\mathbb I^{\mathcal B \times \mathcal U}, \mathbb X^{\mathcal B \times \mathcal U}, \mathbb Y^{\mathcal B \times \mathcal U}, \mathbb Z^{\mathcal B \times \mathcal U}, \mathbb Q^{\mathcal B \times 1}) = \sum_{j=1}^{N_{B}} \sum_{k \in\left\{\mathcal{B} \setminus \{ j \} \right\}} \mathcal{Y}_{B S_{j, k}},
   \end{aligned}
\end{equation}
\noindent where $\mathcal{Y}_{B S_{j, k}}$ represents the SINR of interfered BS $k$ during the communication between the VAA and the $j$th targeted BS.

\par \emph{\textbf{Optimization Objective 3:}} The UAVs will fine-tune their positions to achieve the two above goals, which may result in increased energy consumption. To increase the lifetime of the UAV network, it is essential to reduce the propulsion energy consumption of the UAVs, which is given by
\begin{equation}
  \label{eq:optimization_objective_3}
   \begin{aligned}
     f_3 (\mathbb X^{\mathcal B \times \mathcal U}, \mathbb Y^{\mathcal B \times \mathcal U}, \mathbb Z^{\mathcal B \times \mathcal U}, \mathbb Q^{\mathcal B \times 1}) = \sum_{j=1}^{N_{B}} \sum_{i=1}^{N_{U}} \mathcal{E}_{j,i}(T_{j}),
   \end{aligned}
\end{equation}
\noindent where $T_{j}$ denotes the time spent by the UAV swarm to form a VAA. Moreover, $\mathcal{E}_{j,i}(T_{j}) = \mathcal{P}(V_{j,i}) \times T^{move}_{j,i} + \mathcal{P}(V_{0}) \times (T_{j}-T^{move}_{j,i})$, in which $V_{0}=0$ denotes the UAV is hovering, and $T^{move}_{j,i} = D_{j,i}/V_{j,i}$ means the flight time required for the $i$th UAV for transmitting data to BS $j$.

\par In summary, the MOOP can be expressed as follows:
\begin{subequations}
  \label{MOP-formulation}
  \begin{align}
    \min_{\boldsymbol{X}} \quad &\{f_{1}, -f_{2}, f_{3} \}\\
    \text{s.t.} \quad  & C1: 0 \leqslant  I_{j,i} \leqslant  1, \forall j \in \mathcal{B}, \forall i \in \mathcal{U} \label{eq:const1}\\
    & C2: L_{min} \leqslant x^{U}_{j,i} \leqslant L_{max}, \forall j \in \mathcal{B}, \forall i \in \mathcal{U} \label{eq:const2}\\
    & C3: L_{min} \leqslant y^{U}_{j,i} \leqslant L_{max}, \forall j \in \mathcal{B}, \forall i \in \mathcal{U} \label{eq:const3}\\
    & C4: H_{min} \leqslant z^{U}_{j,i} \leqslant H_{max}, \forall j \in \mathcal{B}, \forall i \in \mathcal{U} \label{eq:const4}\\
    & C5: \mathbb{Q}^{\mathcal B \times 1} \in \mathcal{Q} \label{eq:const5}\\ 
    & C6: D_{(i_1, i_2)} \geq D_{min},  \forall i_1, i_2 \in \mathcal{U} \label{eq:const6}
  \end{align}
\end{subequations}
\noindent where the constraints C1 to C5 define the range of the decision variables, and the constraint C6 is utilized to ensure the minimum distance between any two UAVs to prevent collisions. Specifically, the constraint in (\ref{eq:const1}) limits the excitation current weight \(I_{j,i}\) for each UAV. Moreover, the constraints in (\ref{eq:const2}), (\ref{eq:const3}), and (\ref{eq:const4}) restrict the 3D coordinate of each UAV to a specified range to ensure that the UAV is within the monitoring area and maintains an efficient flight altitude, where $L_{min}$ and $L_{max}$ denote the minimum and maximum ranges of the area that the UAV can fly in the horizontal direction, and $H_{min}$ and $H_{max}$ represent the minimum and maximum flight altitude of the UAV, respectively. Furthermore, the constraint in (\ref{eq:const5}) limits the sequence of data transmission between the VAA and BSs, where $\mathcal {Q} = \{\mathbb{Q}^{\mathcal B \times 1}_{1}, \mathbb{Q}^{\mathcal B \times 1}_{2}, \dots, \mathbb{Q}^{\mathcal B \times 1}_{N_{B}!}\}$ indicates the set of data transmission orders between the UAV-enabled VAA and $N_{B}$ BSs, and there are $N_{B}!$ possible permutation. In addition, the constraint in (\ref{eq:const6}) denotes that any two UAVs must maintain a minimum distance larger than $D_{min}$ to prevent collisions.

\begin{theorem}
The formulated MOOP shown in Eq. (\ref{MOP-formulation}) is NP-hard.
\end{theorem}
\begin{proof}
We assume that the excitation current weights and positions of the UAVs in $f_3$ are fixed and known, and $f_{3}$ is simplified as follows~\cite{Sun2021a}:
\begin{subequations}
  \label{f_3}
  \begin{align}
    \min_{\mathbb{Q}^{\mathcal B \times 1}} \quad & f_{3} = \sum_{j=1}^{N_{B}-1} E_{Q_{j},Q_{j+1}} \label{eq:simplified_f3}\\
    \text{s.t.} \quad  & \mathbb{Q}^{\mathcal B \times 1} \in \mathcal{Q} \label{eq:const_s1}
  \end{align}
\end{subequations}
\noindent where $E_{Q_{j}, Q_{j+1}}$ is the energy consumed by all UAVs in our considered system to fly in a 2D horizontal plane during the transition from communicating with the $Q_{j}$th BS to communicating with the $Q_{j+1}$th BS.
\par Note that $E$ in Eq. (\ref{eq:simplified_f3}) can be calculated according to Eq. (\ref{eq:energy-2d}). Specifically, we first assume that the UAV moves in the horizontal plane at a constant speed $V_{mr}$, in which $V_{mr}$ refers to the optimal UAV speed that maximizes its total motion distance for any given energy, and it can be calculated using the method described in~\cite{Zeng2019a}. Second, we can further calculate the propulsion energy consumption by using Eq. (\ref{eq:energy-2d}). Finally, $E_{Q_{j},Q_{j+1}}$ can be calculated based on the decision variables $(\mathbb X^{\mathcal B \times \mathcal U}, \mathbb Y^{\mathcal B \times \mathcal U})$ and the propulsion energy consumption calculated above.
\par Based on the above analysis, the simplified $f_{3}$  can be viewed as a traveling salesman problem (TSP). Since TSP is the classic NP-hard problem~\cite{Zeng2018,Reda2022}, the simplified $f_{3}$ is also NP-hard. Thus, the MOOP, as defined in Eq. (\ref{MOP-formulation}), is an NP-hard problem.
\end{proof} 

%
%
\section{The Proposed Method}
\label{sec:the_proposed_method}
\par The conventional NSGA-II has some disadvantages such as poor convergence when solving the formulated MOOP. Therefore, we propose the CNSGA-II algorithm by introducing the chaos theory in this section. In addition, we present a feasible deployment strategy for the proposed CB-based method in practical systems.

\subsection{Traditional NSGA-II}
\label{ssec:traditional_NSGA-II}
\par The NSGA-II has several advantages such as ease of implementation and fast operation. Specifically, NSGA-II draws on the basic ideas of genetic algorithms. It maintains a population $\boldsymbol{P}$ that consists of multiple potential solutions to the problem (\emph{i.e}., $\boldsymbol{X_{1}}, \boldsymbol{X_{2}}, \dots, \boldsymbol{X_{N}}$). During the solving process, the population is first generated and then updated iteratively via the select, crossover, and mutation operators. Fig. \ref{fig:selection_strategy} displays the sketch map of the population update, which is briefly described as follows:
\par \emph{\textbf{First}}, the crossover and mutation operators are adopted to update population $\boldsymbol{P_{t}}$ and generate a set of child individuals $\boldsymbol{Q_{t}}$. Note that crossover and mutation are essential operations to achieve population updates in NSGA-II. Specifically, the crossover operation generates new individuals by replacing and recombining part of the structure of two individuals. This operation can increase population diversity and is the core operation of the NSGA-II. On the other hand, mutation operation refers to updating and replacing part of the structure of an individual to form a new individual. Unlike the crossover operation, which must be used in populations with two or more individuals, the mutation operation can be performed on a single individual. A reasonable and effective crossover and mutation strategies can significantly improve the solving performance of evolutionary computation methods.
\par \emph{\textbf{Second}}, the non-dominated sorting strategy is utilized to hierarchically sort these individuals in $\boldsymbol{P_{t}}$ and $\boldsymbol{Q_{t}}$, ensuring that each individual can be correctly classified into its corresponding Pareto level (such as $F_{1}$ and $F_{2}$ shown in Fig. \ref{fig:selection_strategy}, in which the solutions at level $F_{1}$ are dominant to the solutions at level $F_{2}$). 
\par \emph{\textbf{Finally}}, for solutions at the same Pareto level, we further sort them according to their crowding distance. This is because the crowding distance can measure the distribution density of solutions in the Pareto front (PF). Specifically, a larger crowding distance means that the solutions are sparse in the front, which helps to maintain the diversity of the population. The excellent $N$ solutions are selected from the sorted individuals by non-dominated sorting and crowding distance sorting to form a new population $\boldsymbol{P_{t+1}}$.
\begin{figure}
	\centering
	\includegraphics[width=3.5in]{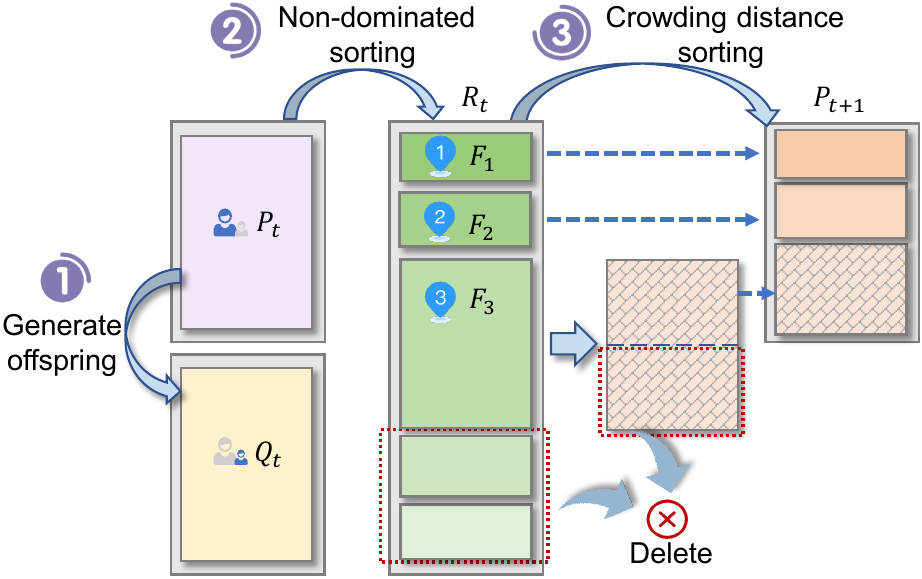}
	\caption{Sketch map of the population update.}
	\label{fig:selection_strategy}
\end{figure}
\begin{algorithm}[]
\caption{CNSGA-II}\label{Algorithm:CNSGA-II}
  \KwIn{population $\boldsymbol{P_{0}}$, population size $N$ and the maximum iteration number $T$, etc.;}
  \KwOut{population $\boldsymbol{P_{T}}$;}
  \tcc{Population initialization stage}
  \For{$n=1$ to $N$}
	{
		Initialize $\boldsymbol X_{n}$ by using Eqs. (\ref{eq:Gauss_map}) and (\ref{eq:initialization})\tcp*{The first improved factor}
        $\boldsymbol{P_{0}} \leftarrow \boldsymbol X_{n}$;\\
	}
    Calculate the objective values for each solution in $\boldsymbol{P_{0}}$;\\
	\For{$t=1$ to $T$}
	{
        \tcc{Solution update stage}
		Select the partial solutions in $\boldsymbol{P_{t}}$ to form parent individuals $\boldsymbol{Parent}$\tcp*{Tournament selection operator}
		Update the solution in $\boldsymbol{Parent}$ by using Algorithm~\ref{Algorithm:Hybrid_Solution_Crossover_Strategy} and Algorithm~\ref{Algorithm:Hybrid_Solution_Mutation_Strategy} to generate the new individuals, denoted as $\boldsymbol{Child}$\tcp*{The second and third improved factors}
        Calculate the objective values for each solution in $\boldsymbol{Child}$;\\
        \tcc{Solution prioritization stage}
		Merge all individuals in $\boldsymbol{Parent}$ and $\boldsymbol{Child}$ to get the merged population $\boldsymbol{R}$;\\
		Delete some solutions by using Algorithm~\ref{Algorithm:elimination_strategy}{\tcp*{The fourth improved factor}}
        Select the optimal $N$ solutions from $\boldsymbol{R}$ as $\boldsymbol{P_{t+1}}$\tcp*{Non-dominated sorting strategy}
	}
	Return $\boldsymbol{P_{T}}$;
\end{algorithm}
\subsection{CNSGA-II}
\label{ssec:CNSGA-II}
\par Given that the formulated problem has a mixed and complex solution space composed of discrete and continuous dimensions, some existing proposed strategies, such as~\cite{Li2024} and~\cite{Li2022a}, are no longer applicable. Thus, we propose the CNSGA-II by introducing the chaotic solution initialization operator, chaos-based hybrid solution crossover and mutation strategies, and improved strategy based on an elimination mechanism to enhance the NSGA-II algorithm, and Algorithm \ref{Algorithm:CNSGA-II} shows the pseudocode of CNASGA-II. Moreover, Fig.~\ref{fig:Framework_CNSGA-II} shows the algorithm framework of CNSGA-II, and the main steps are as follows.
\par \textbf{Step 1. Population Initialization:} CNSGA-II generates the initial population consisting of multiple solutions via our proposed chaotic solution initialization strategy, in which each solution is the potential solution of our formulated MOOP.
\par \textbf{Step 2. Objective Calculation:} The values of three optimization objectives for each solution need to be calculated.
\par \textbf{Step 3. Solution Update:} The algorithm first selects some solutions from the parent population through a tournament selection operator and then applies the proposed chaos-based hybrid solution crossover and mutation operators to update these solutions for generating a child population.
\par \textbf{Step 4. Solution Prioritization:} CNSGA-II merges parent and child populations and employs our proposed elimination strategy and the fast non-dominated sorting method to select the best $N$ solutions for the next generation, denoted as $P_{k+1}$.
\par \textbf{Step 5. Termination Check:} If the algorithm satisfies the predefined termination condition (such as reaching the maximum iteration number), the values in $P_{k+1}$ are determined as the final solutions. Otherwise, the process returns to Step 3 for further iterations.
\par The details of the proposed improved strategies are as follows.
\begin{figure*}
	\centering
	\includegraphics[width=7in]{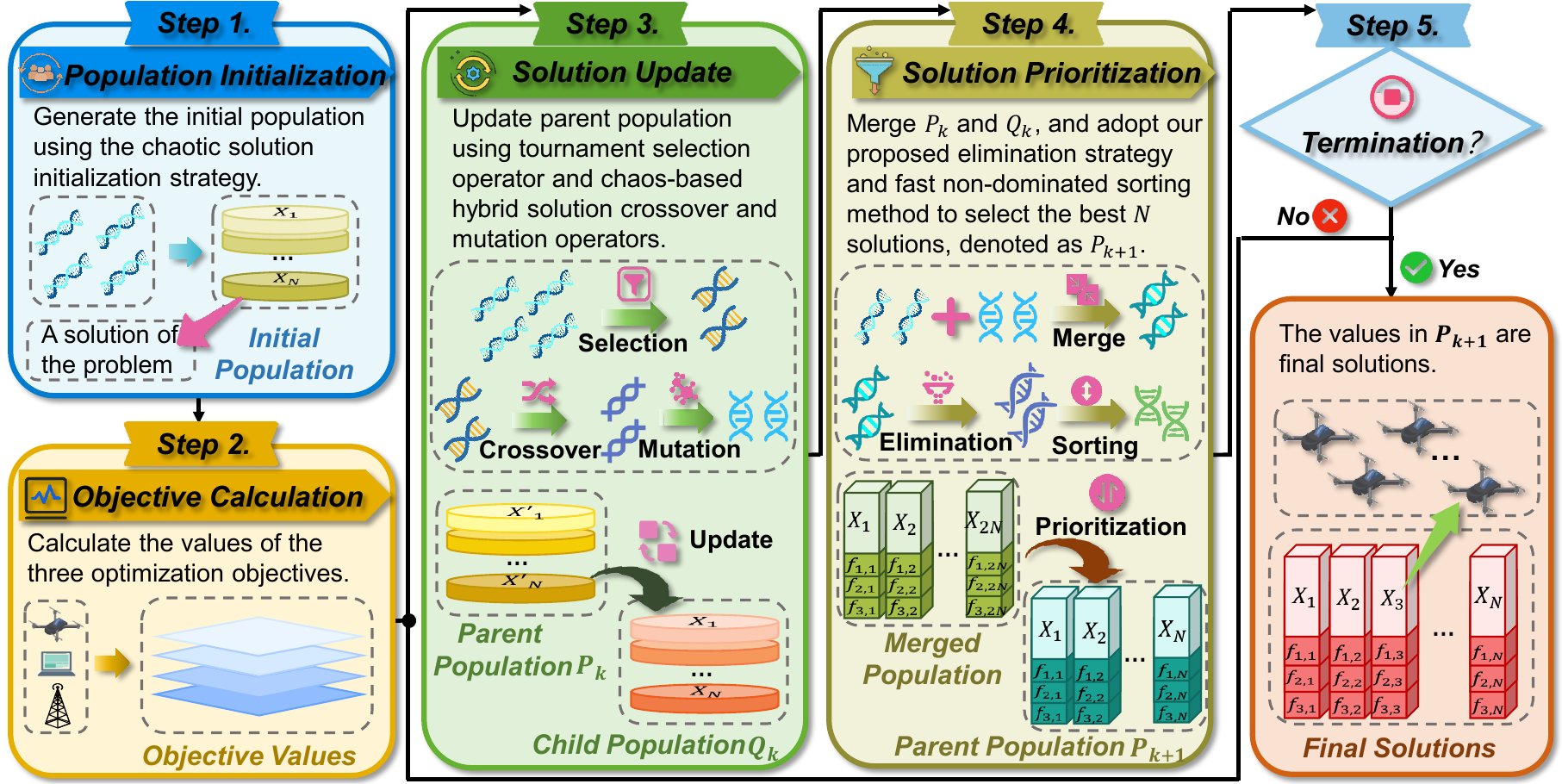}
	\caption{The algorithm framework of CNSGA-II.}
	\label{fig:Framework_CNSGA-II}
\end{figure*}

\par In this section, CNSGA-II introduces chaotic methods in solution initialization, crossover and mutation operations to enhance the solving ability of NSGA-II. Moreover, CNSGA-II uses an improved strategy based on an elimination mechanism to exclude non-competitive solutions to avoid the over-exploration of the search space by the algorithm.

\subsubsection{Chaotic Solution Initialization Operator}
\label{sssec:the_first_operator}
\par The solution initialization is an important step in evolutionary computation methods since it can determine the initial search direction of the algorithm. However, conventional NSGA-II generates the initial solutions in a random way. In this case, the conventional NSGA-II may fall into a local optimum when dealing with the huge solution space of the formulated MMOP. Accordingly, it is essential to enhance the solution initialization procedure of the NSGA-II.
\par Accordingly, chaos theory is an effective approach for improving the quality of the initial solutions~\cite{Li2021}. By mapping the variables in the search space to the chaotic domain, chaotic methods can obtain well-distributed initial solutions. As a result, the Gauss/mouse chaotic map is adopted to optimize the initial population, which can be expressed as follows~\cite{Ewees2020}:
\begin{equation}
  \label{eq:Gauss_map}
  \begin{aligned}
	  {G}_{k+1}= \left\{\begin{array}{ll}
	  1 & {G}_{k} = 0 \\
	  \frac{1}{mod({G}_{k}, 1)} &  \text { otherwise } \\
	  \end{array}\right.,
  \end{aligned}
\end{equation}
\noindent where $G_{k}$ is the $k$th element of the chaotic array generated by the Gauss/mouse map, and the chaotic solution initialization operator is given by
\begin{equation}
  \label{eq:initialization}
   \begin{aligned}
     X^m_{n} = lb^{m} + G_{k} \times (ub^{m} - lb^{m}),
   \end{aligned}
\end{equation}
\noindent where $X^m_{n}$ indicates the $m$th dimension of the $n$th solution, and $ub^{m}$ and $lb^{m}$ denote the upper and lower bounds of the $m$th dimension of a solution, respectively.

\subsubsection{Chaos-based Hybrid Solution Crossover Strategy}
\label{sssec:the_second_operator}
\par The crossover process is the core operation of evolutionary computation methods since it can achieve the update of the solutions~\cite{Li2023a}. However, the crossover operator of traditional NSGA-II may be inapplicable to the formulated problem, resulting in premature convergence of the algorithm. Moreover, this crossover operator cannot achieve the update of the discrete dimensions. Thus, we need to make some improvements in the crossover process of NSGA-II.
\par \textbf{For the continuous part}, the crossover operation of the traditional NSGA-II can be used directly to update continuous dimensions. However, NSGA-II may fall into a local optimum when solving the formulated MOOP with a large-scale optimization space. Accordingly, the crossover operation needs to be improved to enhance the exploration ability of the algorithm. Specifically, based on the simulated binary crossover method, we use the elements of the chaotic sequence generated by the Logistic map as a threshold to determine the generation way of the value $\beta$. The chaos-based crossover strategy can be expressed as follows:
\begin{equation}
  \label{eq:SBX}
   \begin{aligned}
     \left\{\begin{array}{l}C_{t, m}^{1}=0.5 \times \left[\left(1+\beta_{t, m}\right) \cdot P_{t, m}^{1}+(1-\beta_{t, m}) \cdot P_{t, m}^{2}\right] \\C_{t, m}^{2}=0.5 \times \left[\left(1-\beta_{t, m}\right) \cdot P_{t, m}^{1}+(1+\beta_{t, m}) \cdot P_{t, m}^{2}\right]\end{array}\right.,
   \end{aligned}
\end{equation}
\noindent where $t$ is iteration number, and $m$ is the index of a solution dimension. Moreover, $\boldsymbol{P^1} = (p^1_{1}, p^1_{2}, \dots, p^1_{D})$ and $\boldsymbol{P^2} = (p^2_{1}, p^2_{2}, \dots, p^2_{D})$ are the two parent individuals, and $\boldsymbol{C^1} = (c^1_{1}, c^1_{2}, \dots, c^1_{D})$ and $\boldsymbol{C^2} = (c^2_{1}, c^2_{2}, \dots, c^2_{D})$ indicate the two child individuals, in which $D$ denotes the solution dimension size. In addition, $\beta_{t, m}$ is determined dynamically and randomly by the distribution factor $\eta > 0$, which is expressed by
\begin{equation}
  \label{eq:beta}
   \begin{aligned}
     \beta_{t, m}=\left\{\begin{array}{ll}\left(2 u_{t, m}\right)^{\frac{1}{\eta+1}}, & u_{t, m} \leq L_{t} \\\left(\frac{1}{2\left(1-u_{t, m}\right)}\right)^{\frac{1}{\eta+1}}, & \text { otherwise }\end{array}\right.
   \end{aligned}
\end{equation}
\noindent where $L_{t} \in (0,1)$ is the $t$th element of the Logistic chaotic array, and $u_{t, m} \in (0,1)$ is random value.
\par \textbf{For the discrete part}, $\mathbb{Q}^{B \times 1}$ indicates the communication order between the VAA and different BSs. As shown in Eq. (\ref{f_3}), the third objective of the MOOP can be regarded as the TSP, which means that a simple crossover operation can lead to repetition and vacancy of the communications between the VAA and multiple BSs. Thus, a partially mapped crossover (PMX) strategy is adopted to achieve the update of the discrete dimensions~\cite{Ting2010}.
\par Algorithm \ref{Algorithm:Hybrid_Solution_Crossover_Strategy} displays the primary steps of the chaos-based hybrid solution crossover strategy.
\begin{algorithm}[]
\caption{Chaos-based Hybrid Solution Crossover Strategy}\label{Algorithm:Hybrid_Solution_Crossover_Strategy}
  \KwIn{$X_{1}$, $X_{2}$;}
  \KwOut{$X_{1}^{'}$, $X_{2}^{'}$;}
    Cross $X_{1}$ and $X_{2}$ to generate $X_{1}^{c}$ and $X_{2}^{c}$ using Eqs. (\ref{eq:SBX}) and (\ref{eq:beta})\tcp*{Crossover on continuous dimensions}
    Cross $X_{1}$ and $X_{2}$ to generate $X_{1}^{d}$ and $X_{2}^{d}$ using PMX strategy\tcp*{Crossover on discrete dimensions}
	$X_{1}^{'}$ = [$X_{1}^{c}, X_{1}^{d}$];\\
	$X_{2}^{'}$ = [$X_{2}^{c}, X_{2}^{d}$];\\
    Return $X_{1}^{'}$, $X_{2}^{'}$;
\end{algorithm}

\subsubsection{Chaos-based Hybrid Solution Mutation Strategy}
\label{sssec:the_third_operator}
\par Another important step of evolutionary computation algorithms is the mutation process, which is used to keep the diversity of the population~\cite{Rauf2023}. However, traditional NSGA-II is proposed for solving the continuous optimization problem, in which the mutation operator is not suitable for dealing with the formulated MOOP with hybrid solution space. Therefore, we use different mutation strategies to update the continuous and discrete dimensions of the solutions.
\par \textbf{For the continuous part}, similar to the crossover process, we adopt the chaotic mapping to optimize the mutation strategy of the traditional NSGA-II. Specifically, different strategies are selected to achieve the mutation of solutions according to the value $\Delta_{t, j}$ produced by the Chebyshev chaotic array. The chaos-based mutation strategy is given by
\begin{equation}
  \label{eq:PM}
   \begin{aligned}
     C^1_{t,m} = P^1_{t,m} + \Delta_{t, m},
   \end{aligned}
\end{equation}
\noindent where $\Delta_{t, m}$ is calculated as follows:
\begin{equation}
  \label{eq:Delta}
   \begin{aligned}
     \Delta_{t, m}=\left\{\begin{array}{ll}\left(2 u_{t,m}\right)^{\frac{1}{\eta_{u}+1}}-1, & u_{t,m}<C_{t} \\
     1-\left[2\left(1-u_{t,m}\right)\right]^{\frac{1}{1+\eta_{u}}}, & \text { otherwise } \end{array}\right.,
   \end{aligned}
\end{equation}
\noindent where $C_{t}$ is the $t$th element of the Chebyshev chaotic sequence, and $\eta_{u}$ is a user-defined arbitrary non-negative real number.
\par \textbf{For the discrete part}, the exchange mutation (EM) strategy is adopted~\cite{Alipour2018}, and Algorithm \ref{Algorithm:Hybrid_Solution_Mutation_Strategy} displays the primary steps of the chaos-based hybrid solution mutation strategy.
\begin{algorithm}[]
\caption{Chaos-based Hybrid Solution Mutation Strategy}\label{Algorithm:Hybrid_Solution_Mutation_Strategy}
    \KwIn{$X_{1}$;}
  \KwOut{$X_{1}^{'}$;}
    Mutate $X_{1}$ to generate $X_{1}^{c}$ using Eqs. (\ref{eq:PM}) and (\ref{eq:Delta})\tcp*{Mutation on continuous dimensions}
    Mutate $X_{1}$ to generate $X_{1}^{d}$ using EM strategy\tcp*{Mutation on discrete dimensions}
	$X_{1}^{'}$ = [$X_{1}^{c}, X_{1}^{d}$];\\
    Return $X_{1}^{'}$;
\end{algorithm}

\subsubsection{Improved Strategy Based on an Elimination Mechanism}
\label{sssec:the_fourth_operator}
\par NSGA-II may produce some non-competitive solutions while dealing with the formulated MOOP, and these solutions may result in a decrease in the convergence accuracy of the algorithm. Thus, we propose an improved strategy based on an elimination mechanism in this work, and Algorithm \ref{Algorithm:elimination_strategy} displays the pseudocode of the proposed improved operators based on an elimination mechanism.
\begin{algorithm}[]
\caption{The Improved Operator Based on an Elimination Mechanism}\label{Algorithm:elimination_strategy}
  \KwIn{current population $\boldsymbol{P}$;}
  \KwOut{updated population $\boldsymbol{P_{new}}$;}
    Rank $\boldsymbol{P}$ according to the value of $f_{1}$ of solutions;\\
    Remove the worst $\tau_1$ solutions from $\boldsymbol{P}$, and get the new population $\boldsymbol{P_{t}}$;\\
    Rank $\boldsymbol{P_{t}}$ according to the value of $f_{2}$ of solutions;\\
    Remove the worst $\tau_2$ solutions from $\boldsymbol{P_{t}}$, and get the new population $\boldsymbol{P_{new}}$;\\
    Return $\boldsymbol{P_{new}}$;
\end{algorithm}

\subsubsection{Complexity Analysis}
\label{sssec:complexity_analysis}
\begin{theorem}
The complexity of the CNSGA-II is $\mathcal O(N_{obj} \times N^2)$.
\end{theorem}
\begin{proof}
The computational cost of the CNSGA-II is determined by three main components, \emph{i.e}., the computations about objective function, non-dominated sorting, and crowding distance. Assuming the population capacity is $N$ and the number of objectives is $N_{obj}$. The computational complexity of the objective functions is $\mathcal O(N_{obj} \times N)$, and the complexity cost produced during non-dominated sorting of the population is $\mathcal O(N_{obj} \times N^2)$. Moreover, all solutions on the same front need to be judged again according to their crowding distances to maintain population diversity. In this case, the computational complexity of crowding distance is $\mathcal O(N_{obj} \times N \times logN)$. Therefore, the complexity of the CNSGA-II is $\mathcal O(N_{obj} \times N^2)$.
\end{proof}

\subsection{Deployment Strategy}
\label{ssec:Optimization_Strategy}
\par In the considered practical system, we select a UAV to perform and coordinate the entire optimization process to achieve data synchronization and CSI information acquisition. Fig. \ref{fig:optimization_process} shows the optimization process including three main steps, and the details are introduced as follows.
\begin{figure}
	\centering
	\includegraphics[width=3.5in]{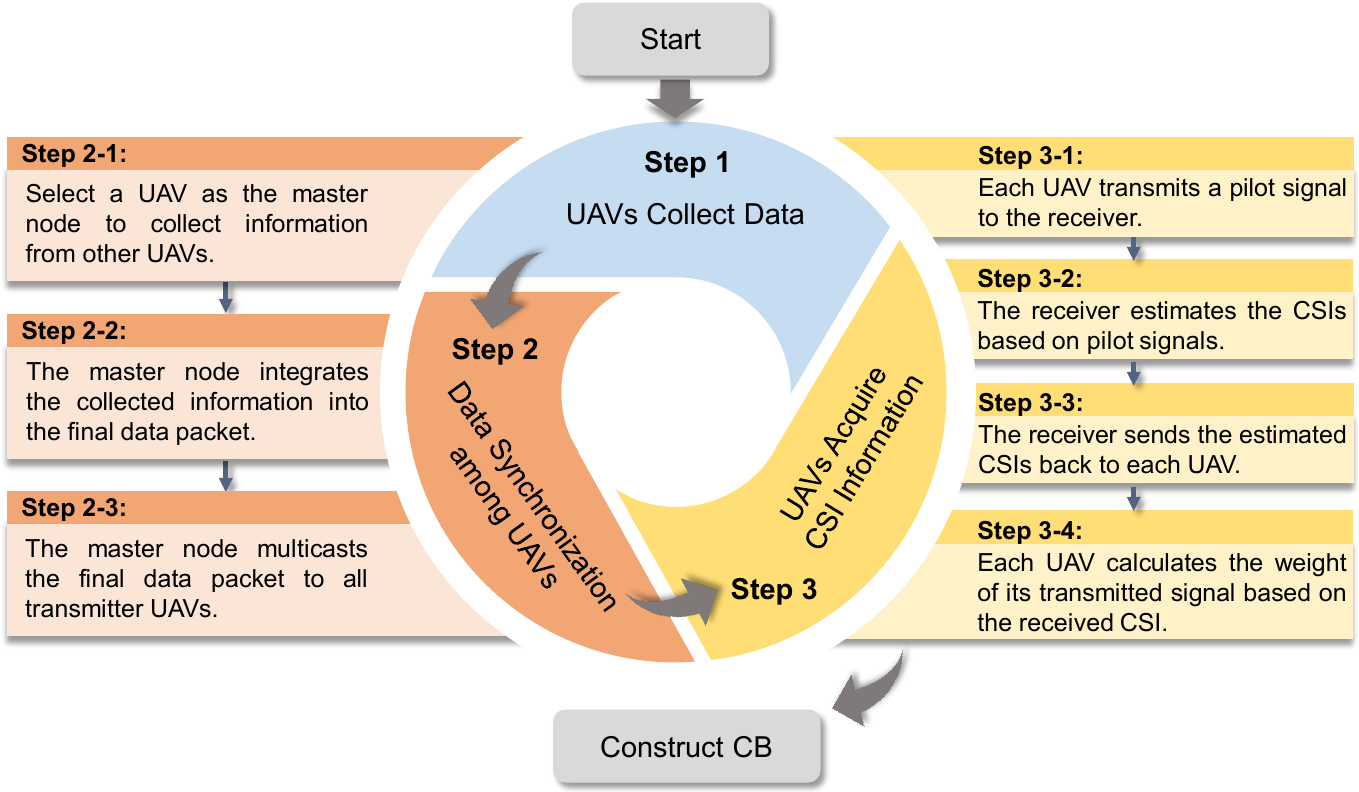}
	\caption{The main steps of deploying our method in practical systems.}
	\label{fig:optimization_process}
\end{figure}
\par \emph{1) UAVs Collect Data:} A group of UAVs are randomly deployed over the monitoring area to collect various environmental data monitored by ground sensors, such as soil moisture and air quality. Note that these UAVs are distributed at different altitudes and positions to ensure comprehensive coverage of the entire monitoring area.
\par \emph{2) Data Synchronization among UAVs:} The UAVs adopt the methods proposed in~\cite{Feng2013} and~\cite{Feng2010} to achieve data sharing. First, a UAV is selected as the master node to collect the information obtained by other UAVs. Second, the master node integrates all data into a final data packet. Finally, the master node multicasts the final data packet to all UAVs. Note that the communication overhead for the data collection and sharing is about 10-20 seconds~\cite{Li2024}, which is negligible compared with the data uplink transmission time.
\par \emph{3) UAVs Acquire CSI Information:} In this step, the UAVs obtain the CSI quantization results using the method proposed in~\cite{Ahmad2022}. First, each UAV sends a known pilot signal to the receiver. Second, the receiver uses the pilot signal to estimate the channel state from each UAV to the receiver, including amplitude, phase, and channel gain information. Third, the receiver sends the estimated CSI back to the UAV through the feedback link. Note that to reduce feedback overhead, CSI is usually quantized before feedback. Finally, the UAV calculates the weight of the transmitted signal based on the received CSI and adjusts the phase of the signal to achieve signal superposition at the receiver.

%
%
\section{Simulation Results and Analysis}
\label{sec:simulation_results_and_analysis}
\par In this section, simulations are conducted on the Matlab platform to assess the performance of CNSGA-II in solving the MOOP shown in Eq. (\ref{MOP-formulation}). Moreover, we also verified the effectiveness of the CB-based interference mitigation approach.
\subsection{Simulation Setups}
\label{ssec:simulation_setups}
\par In our considered simulation scenario, the UAV swarm forms the VAA for transmitting the data collected from the monitoring area (100 m $\times$ 100 m) to multiple ground BSs in sequence. According to the literature~\cite{Sun2021a}, the number of BSs is set to 8, and the number of UAVs is set to 8 and 16. In addition, the flight altitude of UAVs is set between 75 and 95 meters. The reason is that when the UAV flies at this altitude, a more realistic probabilistic LoS channel environment between it and the ground BS can be simulated. Particularly, Table \ref{table:parameters} displays the key parameters in our simulations.
\begin{table}
	\caption{Parameter settings in UAV-enabled data collection network.} \label{table:parameters}
	\begin{tabular}{p{6.3cm} p{1.6cm}}
	\hline
	\bf{Parameter} &\bf{Value} \\
	\hline
	The number of UAVs ($N_{U}$)~\cite{Sun2021a}			& 8, 16\\
	The number of BSs ($N_{B}$)~\cite{Sun2021a}			& 8\\
    The number of GUs 			& 8\\
	Transmission power of the UAV ($P_{U}$)~\cite{Sun2021a}			& 0.1 W\\
    Pathloss exponent ($\alpha$)~\cite{Rappaport2024}			& 3\\
    Transmission power of the GU ($P_{G}$)~\cite{Ho2021}			& 0.1 W\\
	The minimum flight altitude of UAV ($H_{min}$)~\cite{Li2023}			& 75 m\\
	The maximum flight altitude of UAV ($H_{max}$)~\cite{Li2023}			& 95 m\\
	The longest flight distance of UAV ($L_{max}$)~\cite{Li2023}		& 100 m\\
    The weight of UAV ($m_{U}$)~\cite{Sun2021a}        & 2 kg\\
    Blade profile power during UAV hover ($P_1$)~\cite{Zeng2019a} & 79.8563 W \\   
    Blade induced power during UAV hover ($P_2$)~\cite{Zeng2019a} & 88.6279 W \\
    The tip speed of the rotor blade ($v_{tip}$)~\cite{Zeng2019a} & 120 $m/s$ \\
    Average rotor-induced speed during hover ($v_{0}$)~\cite{Zeng2019a} & 4.03 $m/s$ \\
    Fuselage drag ratio ($d_0$)~\cite{Zeng2019a} & 1.225 $kg/m^3$ \\
    Rotor solidity ($s$)~\cite{Zeng2019a} & 0.6 \\
    Rotor disc area ($A$)~\cite{Zeng2019a} & 0.05 \\
    Air density ($\rho$)~\cite{Zeng2019a} & 0.503 $m^2$ \\
	\hline
	\end{tabular}
\end{table}
\par On the other hand, we introduce these evolutionary computation methods for comparison which are multi-objective particle swarm optimization (MOPSO)~\cite{Coello2004}, multi-objective multi-verse optimizer (MOMVO)~\cite{Mrjalili2017}, multi-objective dragonfly algorithm (MODA)~\cite{Mirjalili2016} and traditional NSGA-II~\cite{Deb2002}. Note that some machine learning methods, such as DRL, are not suitable for solving this problem. This is because there may be certain scenario changes (\emph{e.g}., changes in the number of UAVs) in the considered system, which makes the DRL method incur a huge training overhead and difficult to transfer between different scenarios. These comparison algorithms are briefly described as follows:
\begin{itemize}
  \item \emph{\textbf{MOPSO:}} The algorithm adopts a swarm of particles that explores and exploits the search space by updating their velocities and positions under the guise of personal and global best solutions, aiming to find a diverse set of Pareto solutions.
  \item \emph{\textbf{MOMVO:}} The algorithm is inspired by cosmological concepts such as white holes, black holes, and wormholes. It simulates the movement of candidate solutions (universes) through the search space, facilitating exploration and exploitation to identify optimal solutions across multiple objectives.
  \item \emph{\textbf{MODA:}} The algorithm emulates the static and dynamic swarming behaviors of dragonflies, utilizing attraction, alignment, and repulsion mechanisms to search the solution space and converge on a set of optimal trade-off solutions.
  \item \emph{\textbf{NSGA-II:}} The algorithm employs a fast non-dominated sorting approach and a crowding distance mechanism to maintain solution diversity, efficiently guiding the population toward the Pareto-optimal front.
\end{itemize}

\subsection{Results and Analysis}
\label{ssec:results_and_analysis}
\par This work considers two different scale drone-enabled data collection networks consisting of 8 and 16 drones and conducts further analysis and summaries. Through simulations in different scales, the adaptability and flexibility of the proposed method are verified. Moreover, in the larger-scale network, the interference between drones and non-target BSs may increase. Accordingly, simulations in the larger-scale drone-enabled network can be conducted to verify the effectiveness of the CB-based method in complex interference environments.
\subsubsection{The Smaller-scale UAV-enabled Data Collection Network with 8 UAVs}
\par Fig. \ref{fig:results_8UAVs} shows the optimization results obtained by various algorithms in the smaller-scale UAV-enabled data collection network, in which the SINR values provided are cumulative across all interfered BSs to reflect the total interference reduction. It can be observed that the interfered BSs can achieve sufficient SINR for communication, which indicates that the interference caused by the UAV network is effectively mitigated. The reason is that the proposed method enhances beamforming performance by jointly optimizing the excitation current weights and hover positions of UAVs, thus reducing the transmission time and interference to terrestrial communications. Moreover, the proposed CNSGA-II outperforms other classical evolutionary computation methods on three optimization objectives, making it clear that the CNSGA-II is more suited to addressing the MOOP. The reason may be that the proposed improved operators enhance the adaptability of the CNSGA-II, thus enabling the algorithm to achieve an effective balance between exploring potential solution space and exploiting known advantageous solution space. However, compared with the conventional NSGA-II algorithm, the CNSGA-II will incur more computational overhead due to the introduced improved strategies.
\begin{figure}
	\centering
	\includegraphics[width=3.5in]{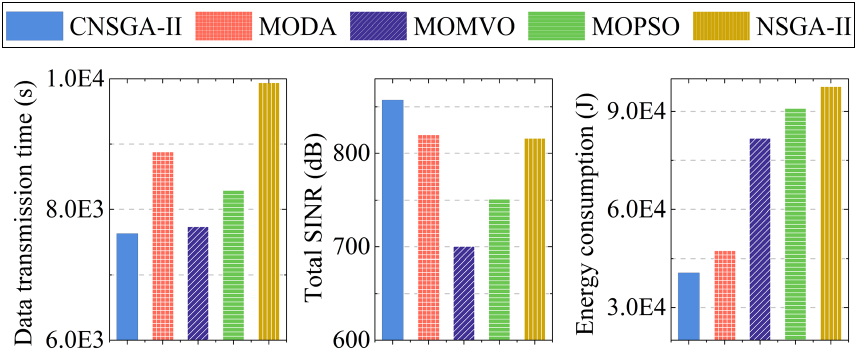}
	\caption{Optimization results obtained by various algorithms (8 UAVs).}
	\label{fig:results_8UAVs}
\end{figure}
\par Fig. \ref{fig:solution_distributions_8UAVs} displays the solution distributions attained by several algorithms. Furthermore, while some classic evolutionary computation methods, such as NSGA-II and MOPSO, can effectively mitigate interference, significantly reducing the transmission and energy efficiency of the UAV-enabled data collection and A2G transmission system, which is not applicable to the scenario considered in this work. Moreover, these solution distributions indicate that the balance among objectives in CNSGA-II is critical and needs to be considered carefully.
\begin{figure}
	\centering
	\includegraphics[width=3.5in]{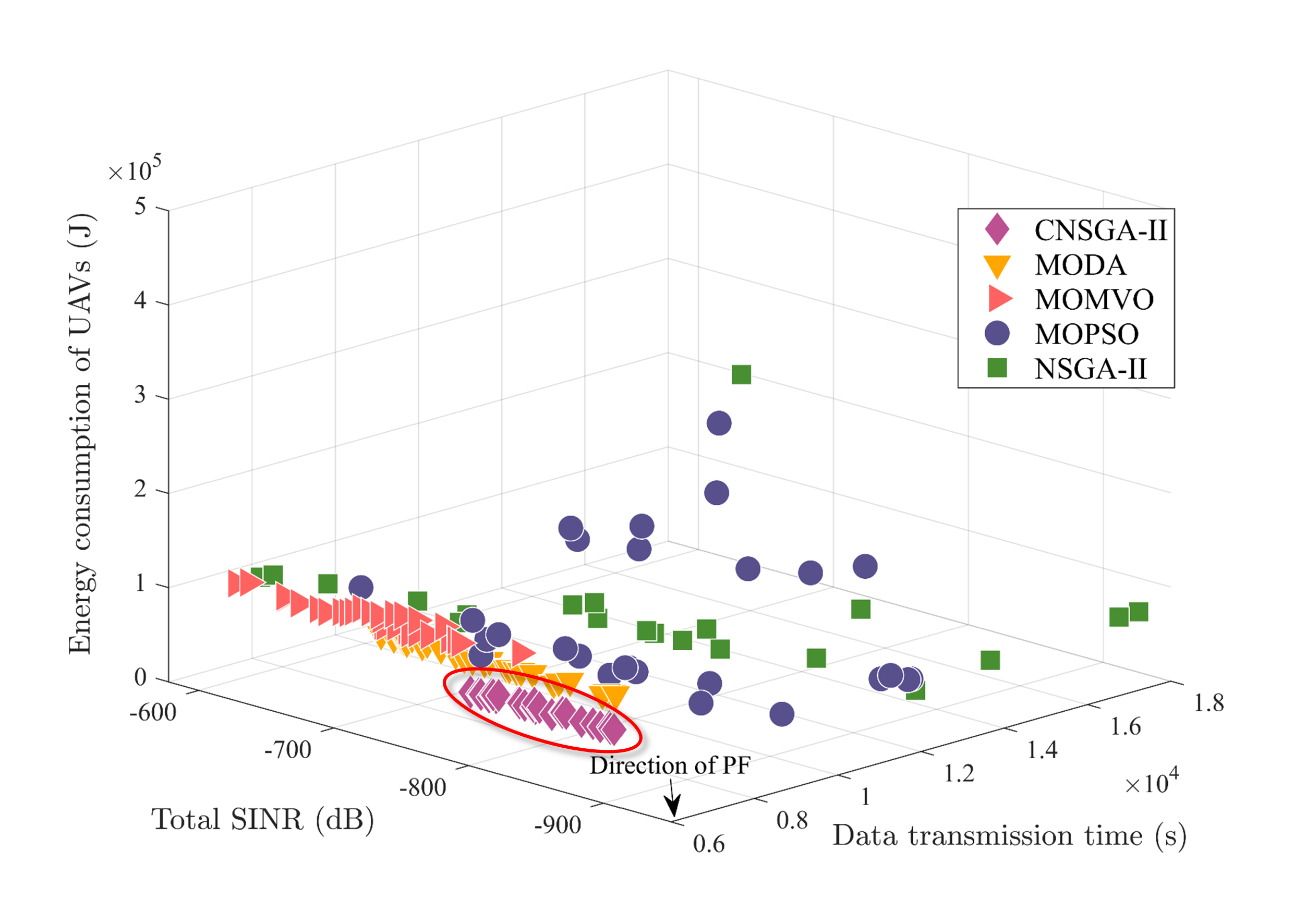}
	\caption{Solution distributions of various algorithms (8 UAVs).}
	\label{fig:solution_distributions_8UAVs}
\end{figure}
\par Moreover, as seen in the figure, the Pareto solutions provided by CNSGA-II are closer to the PF direction, indicating that CNSGA-II has a better solving ability for the MOOP. It can be because the chaos-based hybrid solution crossover and mutation strategies introduced by CNSGA-II can balance the global and local search ability of the algorithm, which leads to achieving a more comprehensive search of the solution space.
\par In addition, the solutions obtained by CNSGA-II are evenly distributed and perform well on all objectives, which means that it achieves a reasonable and efficient exploration of the search space. This is attributed to the chaos-based initialization operator improving the distribution performance and diversity of initial solutions, thus enhancing the initial search performance of traditional NSGA-II. Meanwhile, the improved operator based on an elimination mechanism further improves the capability of the algorithm to perform local searches, preventing the algorithm from failing to converge.
\par Fig. \ref{fig:Tra_8UAVs} gives the flight paths of UAVs obtained by various algorithms in the smaller-scale case for transmitting the collected data with the first BS. It can be seen that the optimized positions obtained by CNSGA-II are closer to the center of the initial positions of UAVs, which indicates that UAVs can generate less propulsion energy consumption.
\begin{figure*}
	\centering
	\includegraphics[width=7in]{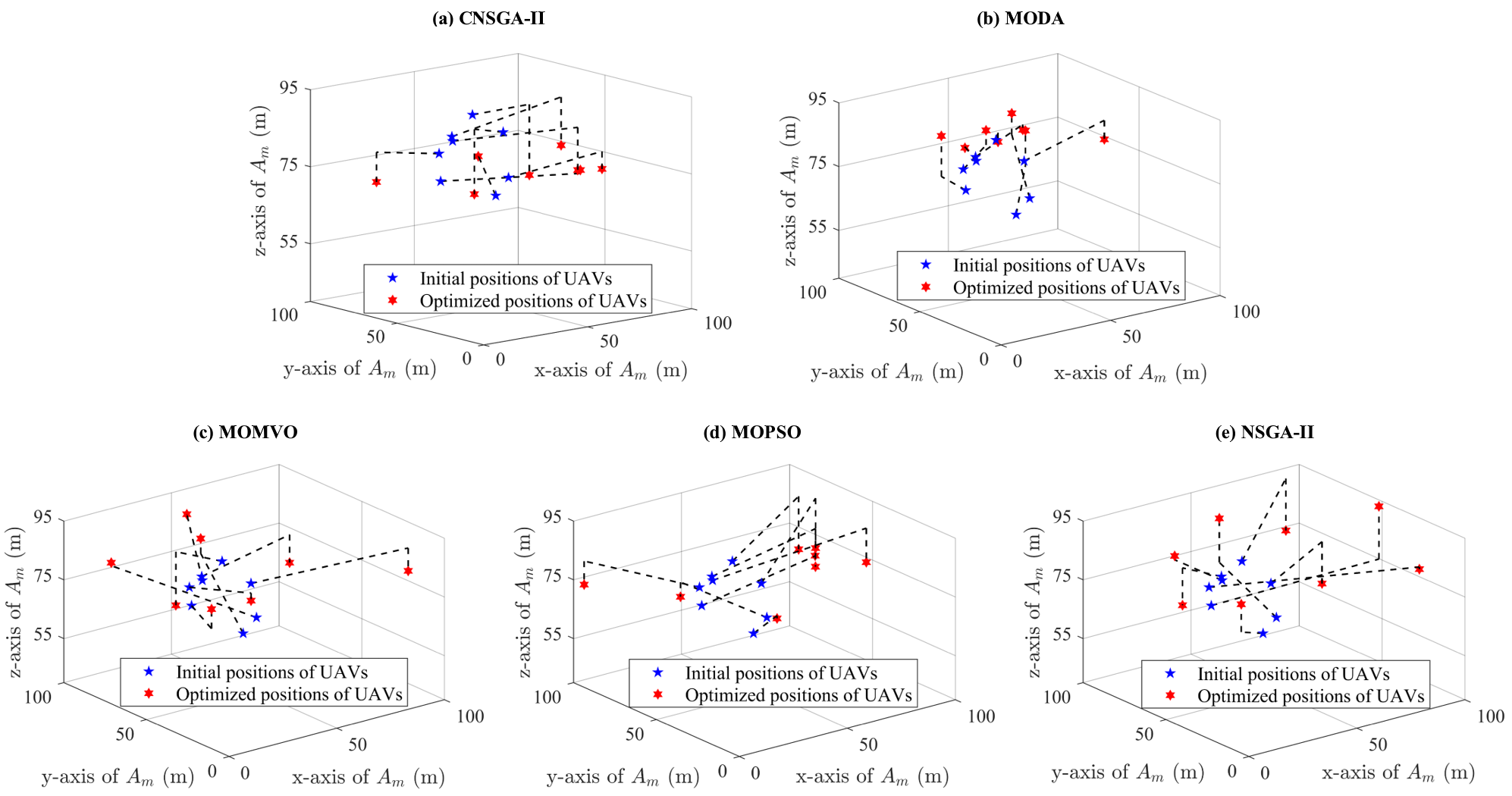}
	\caption{Flight paths of UAVs obtained by various algorithms (8 UAVs). (a) CNSGA-II. (b) MODA. (c) MOMVO. (d) MOPSO. (e) NSGA-II.}
	\label{fig:Tra_8UAVs}
\end{figure*}
\subsubsection{The Larger-scale UAV-enabled Data Collection Network with 16 UAVs}
\par In this section, the number of UAVs is set to 16 while keeping the other parameters in our considered system fixed.
\begin{figure}
	\centering
	\includegraphics[width=3.5in]{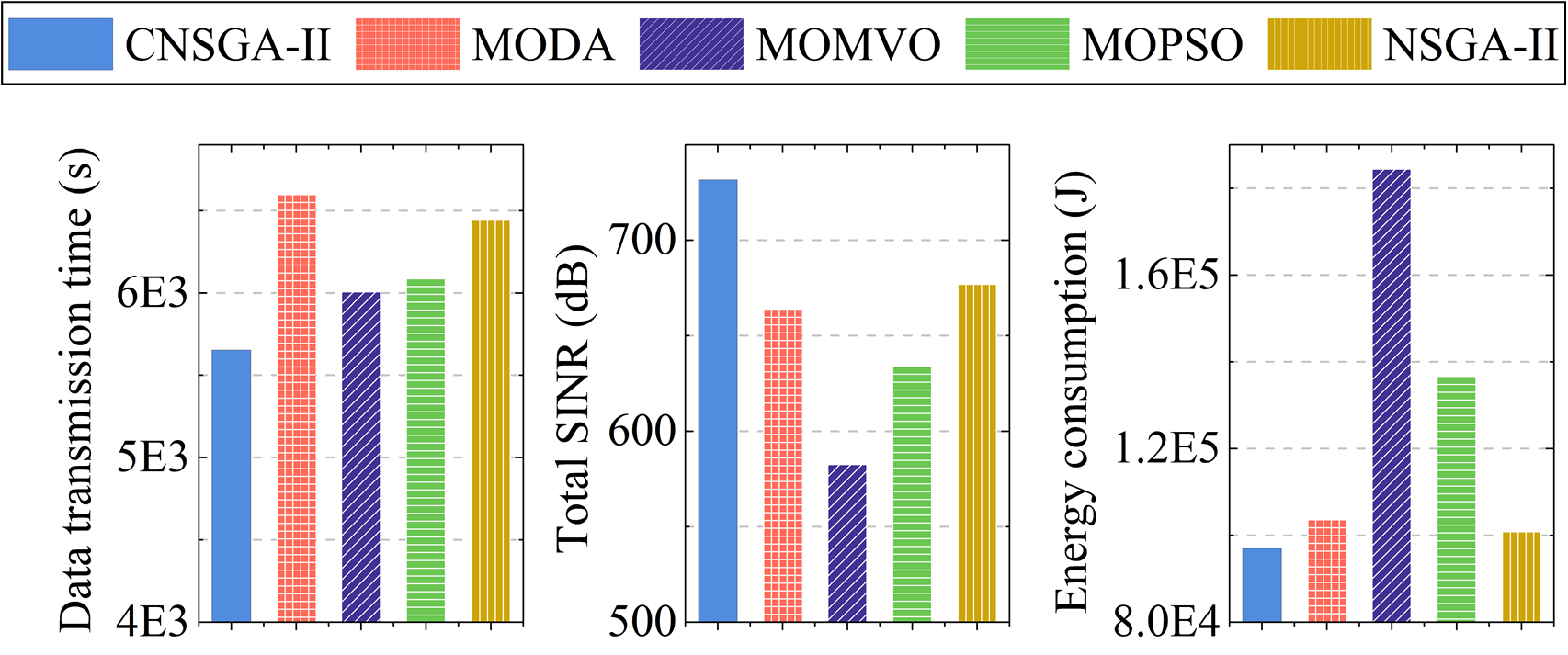}
	\caption{Optimization results obtained by various algorithms (16 UAVs).}
	\label{fig:results_16UAVs}
\end{figure}
\par Fig. \ref{fig:results_16UAVs} shows the optimization results on three objectives obtained by various algorithms in the larger-scale case. The CNSGA-II algorithm still maintains its performance advantage when expanded to a larger-scale UAV-enabled data collection network, which means that the proposed CNSGA-II has validity, superiority, and certain universality. However, compared with the simulation results in the smaller-scale UAV network, when the number of UAVs increases and transmission power remains constant, the signals emitted by each UAV will generate a higher mainlobe by superposition, which can effectively improve the data transmission efficiency of the UAV-enabled VAA. Similarly, the interference of the VAA on terrestrial network devices also will be more severe. In addition, the total energy consumption generated by all UAVs will significantly increase due to the increasing number of UAVs. Moreover, although CNSGA-II maintains effective performance up to the case of 16 UAVs, scalability challenges may emerge in even larger networks due to increased computational complexity and optimization difficulty.
\begin{figure}
	\centering
	\includegraphics[width=3.5in]{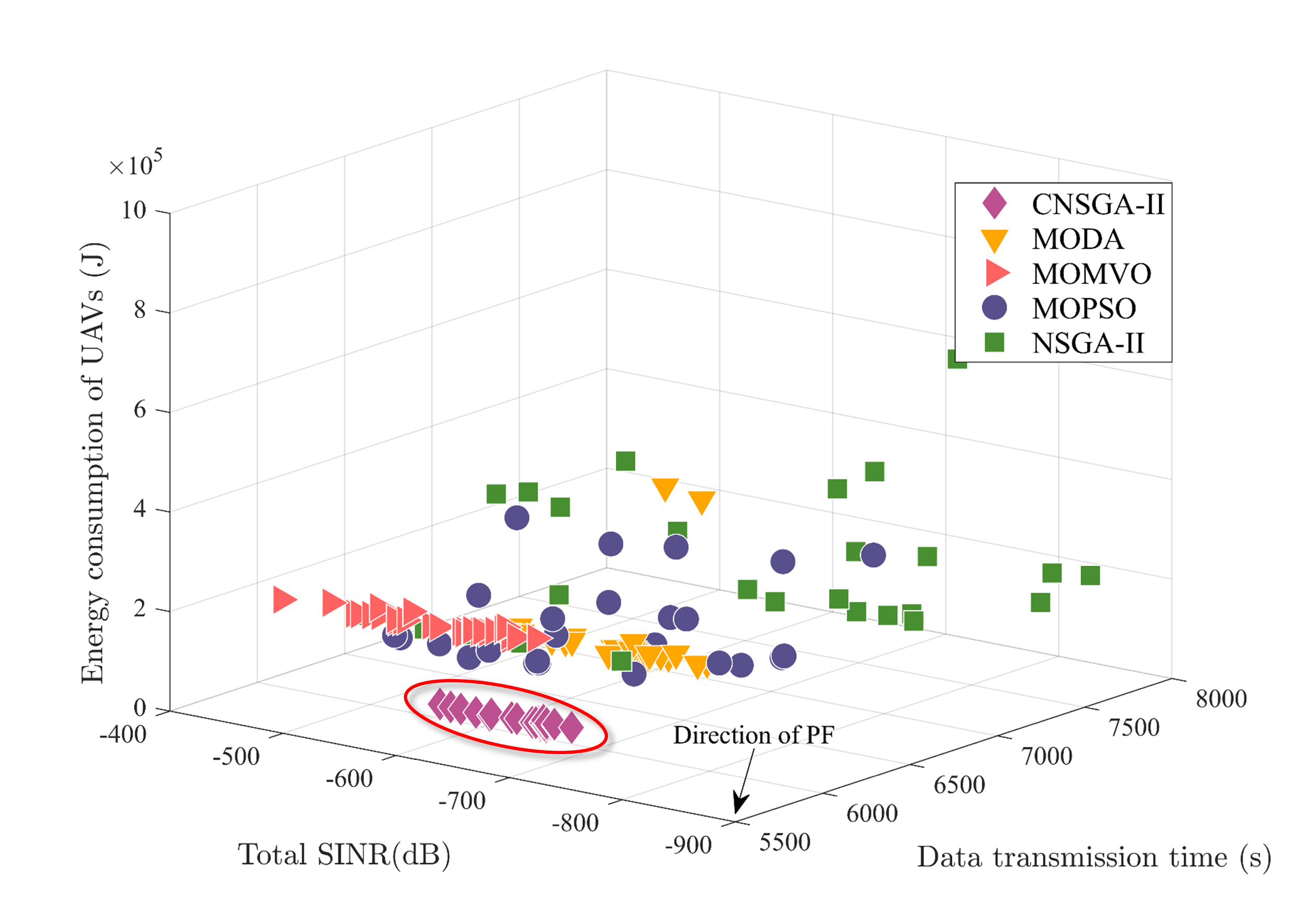}
	\caption{Solution distributions of various algorithms (16 UAVs).}
	\label{fig:solution_distributions_16UAVs}
\end{figure}
\par Fig. \ref{fig:solution_distributions_16UAVs} gives the solution distributions obtained by various algorithms for solving the formulated MOOP. As can be seen, the solutions obtained by our proposed algorithm are closer to the PF and well-distributed. This indicates that the proposed CNSGA-II algorithm is still effective and finds different trade-offs between interference mitigation performance and data transmission efficiency and energy efficiency in the larger-scale case. This is because our proposed chaos-based initialization operator generates higher-quality initial individuals, thus accelerating the optimization process. Furthermore, the chaos-based hybrid solution crossover and mutation strategies as well as the improved operator based on an elimination mechanism in CNSGA-II enhance the global and local search abilities. However, the computational complexity of CNSGA-II increases significantly with larger-scale UAV networks. In this case, the growing search space and optimization complexity will lead to higher computational demands and longer convergence times.
\begin{figure*}
	\centering
	\includegraphics[width=7in]{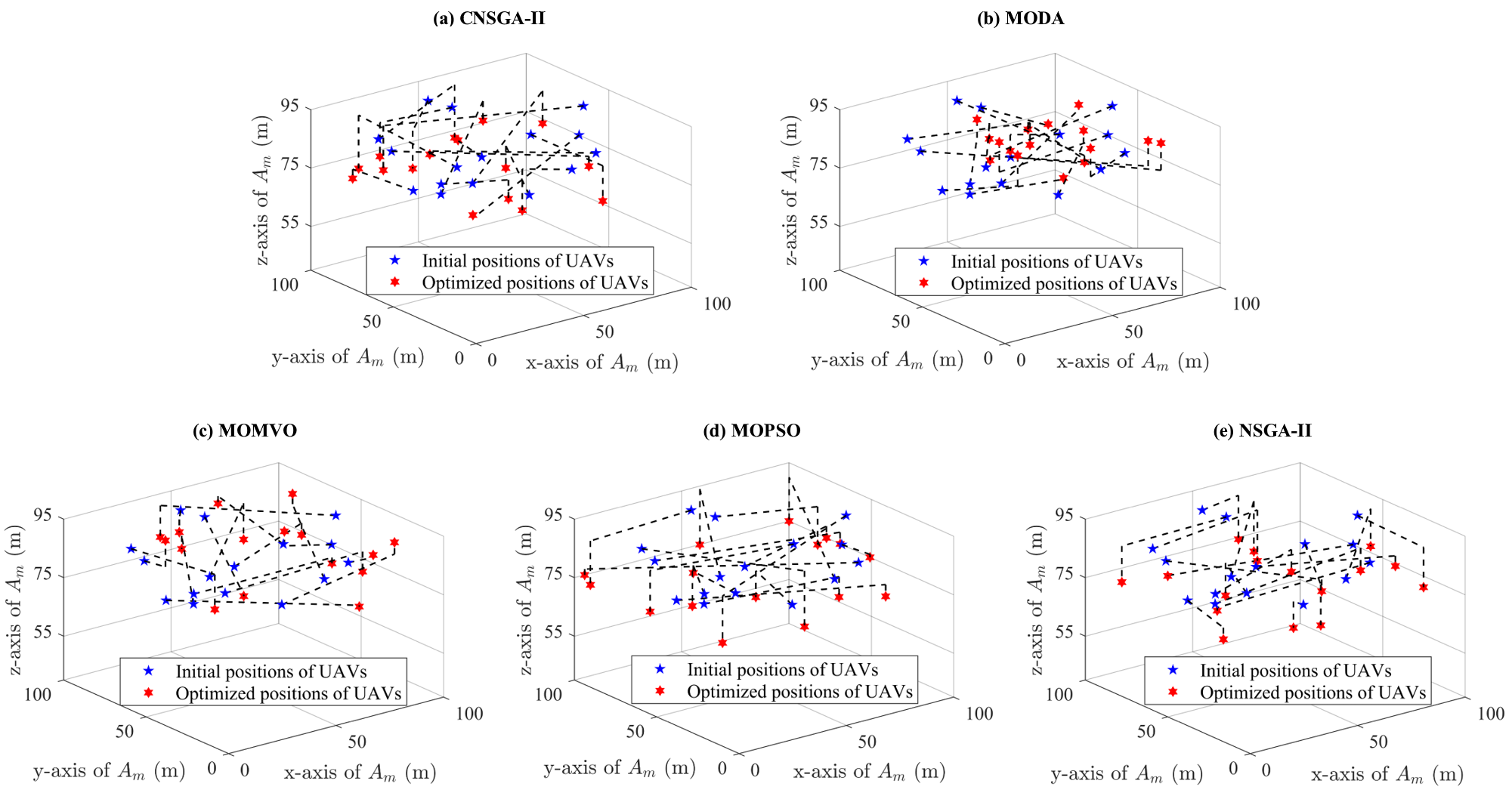}
	\caption{Flight paths of UAVs obtained by various algorithms (16 UAVs). (a) CNSGA-II. (b) MODA. (c) MOMVO. (d) MOPSO. (e) NSGA-II.}
	\label{fig:Tra_16UAVs}
\end{figure*}
\par Fig. \ref{fig:Tra_16UAVs} shows the flight paths of UAVs obtained by various algorithms in the larger-scale case. Note that these UAVs are transmitting the collected data with the first BS. As we can see, the total flight distance obtained by CNSGA-II is relatively shorter, which indicates that UAVs can consume less energy for propulsion. The reason may be that the chaos-based initialization operator improves the distribution of initial solutions, enhancing the initial search performance. However, while clustering UAVs near their initial positions reduces energy consumption in smaller-scale or controlled environments, this approach may not scale effectively in larger or more dynamic networks where wider UAV dispersion is necessary for optimal coverage and performance.

\subsubsection{Performance Verification of the CB-based Interference Mitigation Strategy}
\par This article investigates the communication problem in a UAV-to-ground interference mitigation scenario. In this scenario, GUs need to communicate with BSs, while the A2G transmission between UAVs and the corresponding target BS can cause interference, degrading the signal quality and communication efficiency from GUs to BSs. As described in~\cite{Mozaffari2016} and~\cite{Mei2020}, the interference effect can be quantified by the change in SINR. To evaluate and optimize communication quality, we use the total SINR of communication of GUs with BSs as a metric, where the SINR is cumulative across all interfered BSs.
\par Based on this, we consider two network configurations which are A2G transmission without adopting CB and A2G transmission with CB, respectively. Table~\ref{table:comparison_SINR} presents total SINR optimization results for interfered BSs. It is necessary to emphasize that the high SINR values presented in Table~\ref{table:comparison_SINR} are cumulative across all interfered BSs, reflecting the total interference reduction rather than the SINR of individual links. As can be seen, in the smaller UAV-enabled data collection network with 8 UAVs, the total SINR increases by 4.95 times after applying the CB method. Similarly, in the larger-scale UAV-enabled data collection network with 16 UAVs, total SINR improves by 4.23 times post-CB application. These results demonstrate that the CB-based method effectively reduces interference in the communication process between the UAVs and the target BSs on other communications. This improved performance can be attributed to the CB method which enhances the directivity and gain of the mainlobe and minimizes the sidelobe levels that cause interference by optimizing the excitation current weights and hovering positions of UAVs. However, UAVs may need to hover at specific locations or frequently adjust their positions to achieve optimal beamforming, which leads to additional propulsion energy consumption. Moreover, data sharing and synchronization are required between the UAV swarm in the CB-based method, which may generate some additional communication overhead.
\begin{table}
  \centering
  \caption{The comparison of total SINR of interfered BSs results before and after applying the CB method.} \label{table:comparison_SINR}
  \begin{tabular}{c c c}
  \hline
  \textbf{Scenario}       & \textbf{Before}    & \textbf{After} \\
  \hline
        Smaller-scale     & 173.1319 dB        & 857.2498 dB      \\
        Larger-scale      & 172.7776 dB        & 731.6833 dB      \\ 
  \hline
  \end{tabular}
\end{table}

%
%
\section{Conclusion}
\label{sec:conclusion}
\par This paper has investigated the interference problem affecting non-receivers during UAV uplink transmission. Specifically, we have examined a UAV-enabled data collection and A2G transmission scenario, where UAVs collect data from ground sensors and uplink it to remote BSs, while GUs communicate with these BSs. In this case, the data transmission will interfere with the communications between GUs and BSs. To address this issue, we have formulated the MOOP aimed at minimizing data transmission time, improving the total SINR of interfered BSs, and reducing UAV propulsion energy consumption. Additionally, we have proved that this MOOP is NP-hard and proposed the CNSGA-II with four improved factors to deal with it. Simulation results have indicated that CNSGA-II can generate more excellent solutions compared with other evolutionary computation algorithms. Furthermore, the CB-based interference mitigation method significantly reduces interference, with total SINR results improving by 4.95 times and 4.23 times in two cases with 8 UAVs and 16 UAVs, respectively. This substantial increase in cumulative SINR across interfered BSs has highlighted the efficacy of the CB-based method in mitigating interference for non-receivers.
\par Based on this work, future studies will explore the CB-based interference mitigation method in dynamic network environments, where UAV mobility and channel conditions vary over time. Moreover, advanced machine learning techniques, such as generative artificial intelligence, can be integrated to optimize UAV trajectories, thereby enhancing the performance of UAV networks. Furthermore, the optimization framework can be extended to include UAV battery management strategies for enhancing the sustainability and practicality of UAV communications. In addition, we will further explore the combination of the CB method and MA array for designing a novel hybrid beamforming method. These research directions will enable a deeper understanding of interference challenges in UAV networks while improving overall network efficiency and reliability.

\bibliographystyle{IEEEtran}
\bibliography{myref}

\begin{IEEEbiography}[{\includegraphics[width=1in,height=1.25in, clip,keepaspectratio]{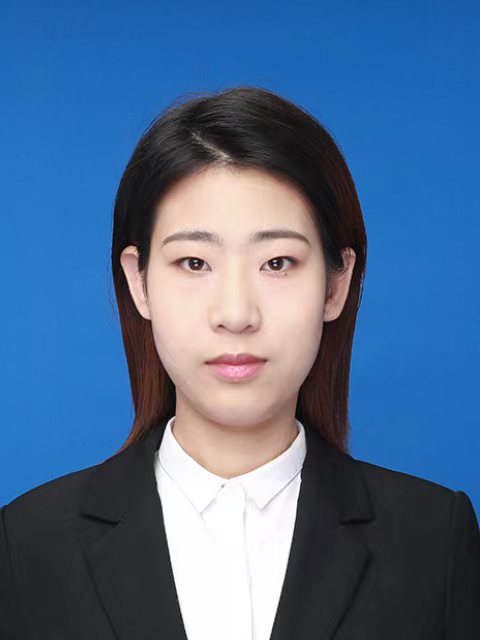}}]{Hongjuan Li} received a BS degree in Computer Science and Technology from Hebei Normal University, Shijiazhuang, China, in 2020, and the M.S. degree in Software Engineering from Jilin University, Changchun, China in 2023. She is currently studying Computer Science at Jilin University to get a Ph.D. degree. Her current research focuses on UAV networks and optimization.
\end{IEEEbiography}

\vspace{-10 mm}
\begin{IEEEbiography}[{\includegraphics[width=1in,height=1.25in,clip,keepaspectratio]{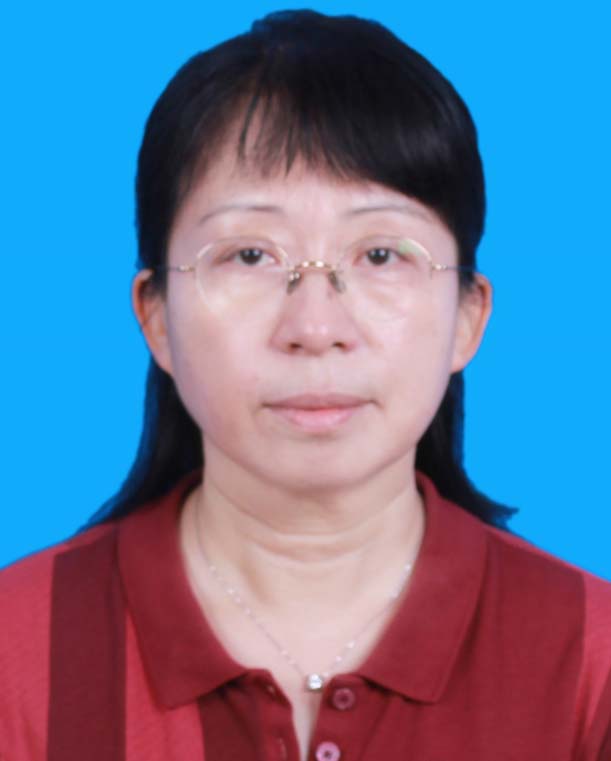}}]{Hui Kang} received the M.E. and Ph.D. degrees from Jilin University in 1996 and 2007, respectively. She is currently a Professor with the College of Computer Science and Technology, Jilin University. Her research interests include information integration and distributed computing.
\end{IEEEbiography}

\vspace{-10 mm}
\begin{IEEEbiography}[{\includegraphics[width=1in,height=1.25in,clip,keepaspectratio]{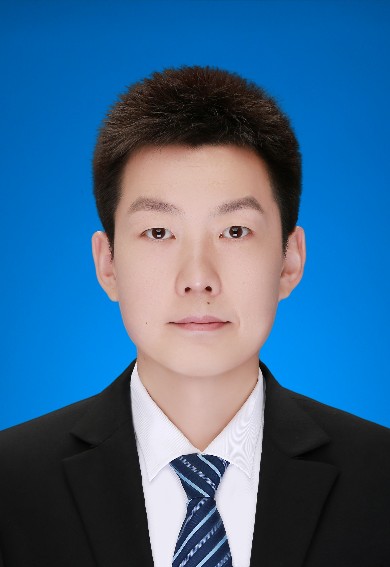}}]{Geng Sun} (Senior Member, IEEE) received the B.S. degree in communication engineering from Dalian Polytechnic University, and the Ph.D. degree in computer science and technology from Jilin University, in 2011 and 2018, respectively. He was a Visiting Researcher with the School of Electrical and Computer Engineering, Georgia Institute of Technology, USA. He is a Professor in College of Computer Science and Technology at Jilin University, and His research interests include wireless networks, UAV communications, collaborative beamforming and optimizations.
\end{IEEEbiography}

\vspace{-10 mm}
\begin{IEEEbiography}
[{\includegraphics[width=1in,height=1.25in,clip,keepaspectratio]{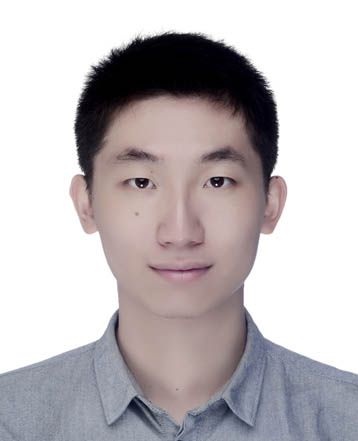}}]
{Jiahui Li} received his B.S. in Software Engineering, and M.S. and Ph.D. in Computer Science and Technology from Jilin University, Changchun, China, in 2018, 2021, and 2024, respectively. He was a visiting Ph.D. student at the Singapore University of Technology and Design (SUTD). He currently serves as an assistant researcher in the College of Computer Science and Technology at Jilin University. His current research focuses on integrated air-ground networks, UAV networks, wireless energy transfer, and optimization.
\end{IEEEbiography}

\vspace{-10 mm}
\begin{IEEEbiography}[{\includegraphics[width=1in,height=1.25in,clip,keepaspectratio]{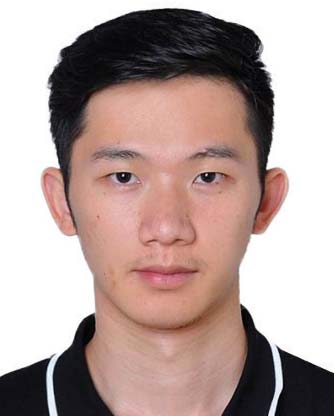}}]{Jiacheng Wang} received the Ph.D. degree from the School of Communication and Information Engineering, Chongqing University of Posts and Telecommunications, Chongqing, China. He is currently a Research Associate in computer science and engineering with Nanyang Technological University, Singapore. His research interests include wireless sensing, semantic communications, and metaverse.
\end{IEEEbiography}

\vspace{-10 mm}
\begin{IEEEbiography}[{\includegraphics[width=1in,height=1.25in,clip,keepaspectratio]{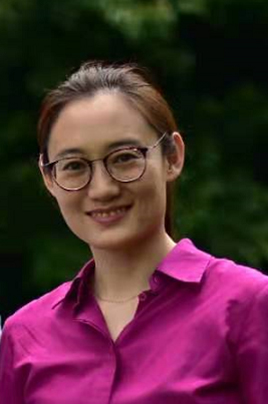}}]{Xue Wang} (Senior Member, IEEE) received the M.E. degree in Communication and Information Systems from Jilin University, China, in 2009, the Ph.D. degree in Communication and Information Systems from Jilin University, China, in 2012. Since 2021, she has been a Full Professor with the Department of Communication Engineering, Jilin University, China. Her research focuses on key technologies in 5/6G communications, the application of artificial intelligence in the next generation communication system, large scale heterogeneous dense access and multi-access edge computing.
\end{IEEEbiography}

\vspace{-10 mm}
\begin{IEEEbiography}[{\includegraphics[width=1in,height=1.25in,clip,keepaspectratio]{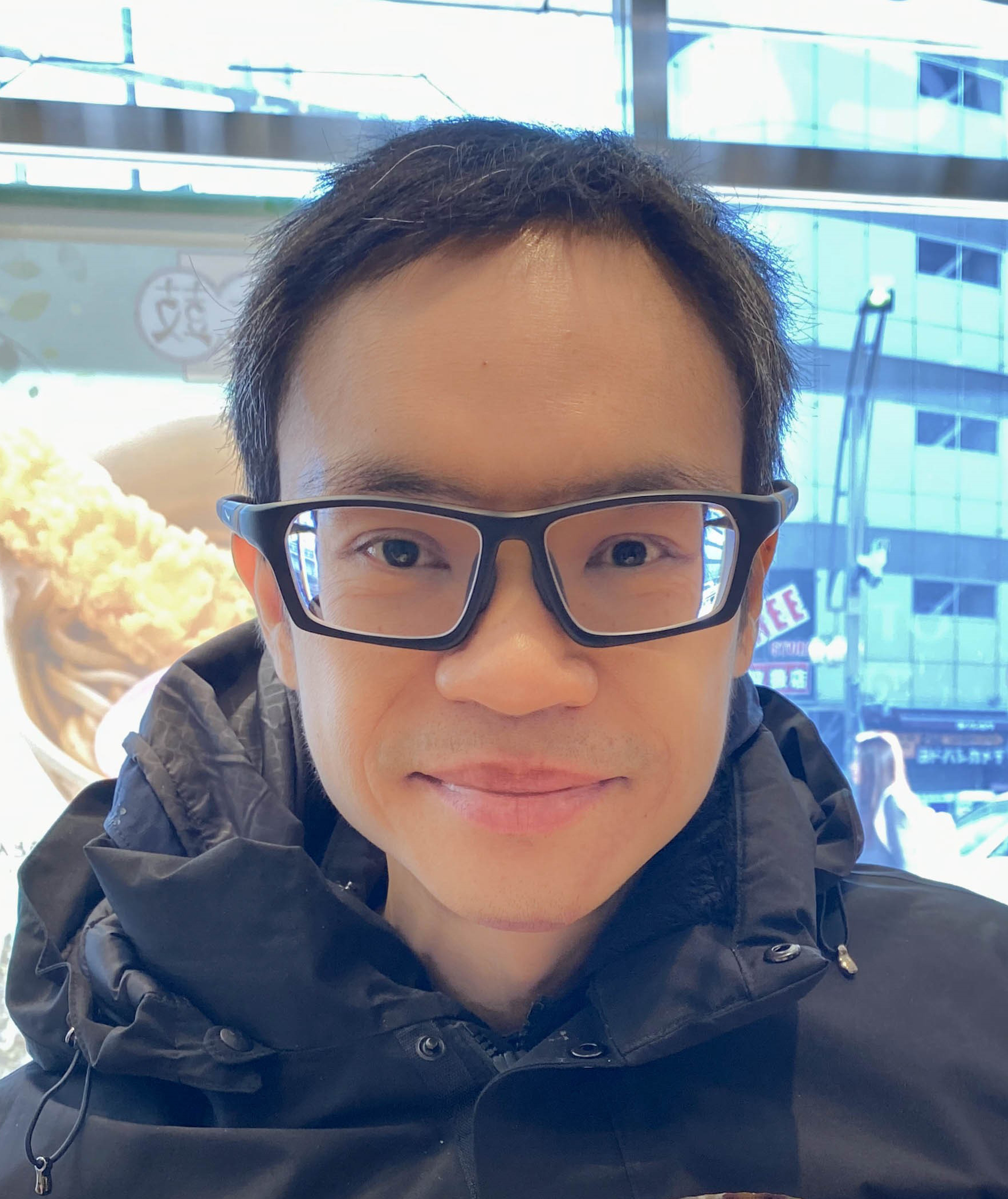}}]{Dusit Niyato} (Fellow, IEEE) is a professor in the College of Computing and Data Science, at Nanyang Technological University, Singapore. He received B.Eng. from King Mongkuts Institute of Technology Ladkrabang (KMITL), Thailand and Ph.D. in Electrical and Computer Engineering from the University of Manitoba, Canada. His research interests are in the areas of sustainability, edge intelligence, decentralized machine learning, and incentive mechanism design.
\end{IEEEbiography}

\vspace{-10 mm}
\begin{IEEEbiography}[{\includegraphics[width=1in,height=1.25in,clip,keepaspectratio]{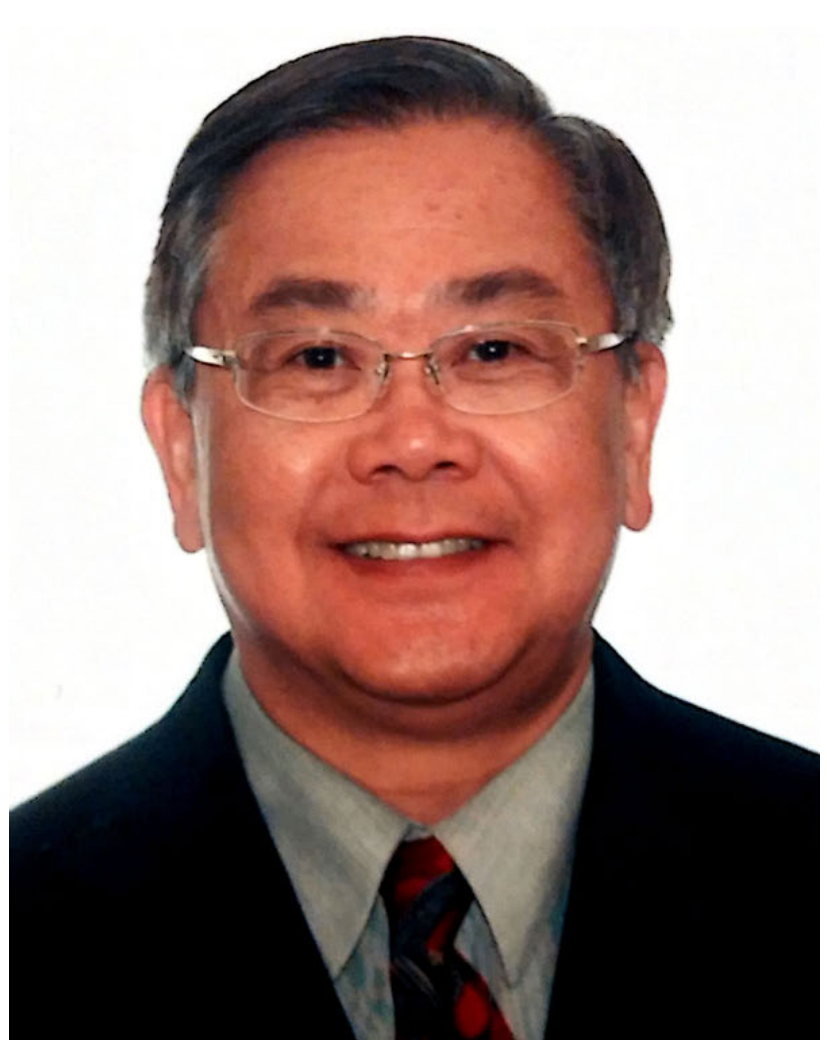}}]{Victor C. M. Leung} (Life Fellow, IEEE) is currently affiliated with Shenzhen MSU-BIT University and was a Distinguished Professor of Computer Science and Software Engineering with Shenzhen University, China. He is also an Emeritus Professor of electrial and computer engineering and the Director of the Laboratory for Wireless Networks and Mobile Systems at the University of British Columbia (UBC). His research is in the broad areas of wireless networks and mobile systems. He has co-authored more than 1300 journal/conference papers and book chapters. Dr. Leung is serving on the editorial boards of \textsc{IEEE Transactions on Green Communications and Networking}, \textsc{IEEE Transactions on Cloud Computing}, \textsc{IEEE Access}, and several other journals.
\end{IEEEbiography}
\end{document}